\documentclass[twoside]{article}

\usepackage[accepted]{aistats2020}

\usepackage[round]{natbib}

\usepackage{amsmath,amsfonts,stmaryrd,amsthm,amssymb,enumitem}
\usepackage{algorithm,booktabs,url,hyperref}
\usepackage[color=yellow]{todonotes}

\usepackage{algorithmic}

\usepackage{pifont}

\usepackage{tikz}
\usetikzlibrary{arrows}
\usetikzlibrary{calc}
\usetikzlibrary{shapes}
\usetikzlibrary{fit}
\usetikzlibrary{matrix} 
\usetikzlibrary{positioning}
\usetikzlibrary{decorations.pathreplacing}

\newtheorem{theorem}{Theorem}[section]
\newtheorem{lemma}[theorem]{Lemma}

\newtheorem{definition}[theorem]{Definition}

\renewcommand{\eqref}[1]{Eq.~(\ref{eq:#1})}
\newcommand{\figref}[1]{Figure~\ref{fig:#1}}
        
\newcommand{\secref}[1]{Section~\ref{sec:#1}}
\newcommand{\thmref}[1]{Theorem~\ref{thm:#1}}
\newcommand{\lemref}[1]{Lemma~\ref{lem:#1}}

\newcommand{\algref}[1]{Alg.~\ref{alg:#1}}

\renewcommand{\P}{\mathbb{P}}
\newcommand{\E}{\mathbb{E}}
\newcommand{\Var}{\mathrm{Var}}

\newcommand{\nats}{\mathbb{N}}
\newcommand{\half}{{\frac12}}

\newcommand{\one}{\mathbb{I}}

\DeclareMathOperator*{\argmin}{argmin}

\newcommand{\cA}{\mathcal{A}}

\newcommand{\cX}{\mathcal{X}}

\usepackage{todonotes}

\newcommand{\algname}{\textsf{SKM}}

\newcommand{\OPT}{\mathrm{OPT}}

\newcommand{\ball}{\mathrm{Ball}}
\newcommand{\optalg}{\textsf{SKM2}}
\newcommand{\rsk}{r}

\begin{document}

\newcommand{\papertitle}{Sequential no-Substitution $k$-Median-Clustering}

\twocolumn[

\aistatstitle{\papertitle}

\aistatsauthor{ Tom Hess \And Sivan Sabato }

\aistatsaddress{ Department of Computer Science \\
Ben-Gurion University of the Negev\\
Beer Sheva 8410501, Israel \And  Department of Computer Science \\
Ben-Gurion University of the Negev\\
Beer Sheva 8410501, Israel
} ]

\begin{abstract}

  We study the sample-based $k$-median clustering objective under a sequential
  setting without substitutions. In this setting, an i.i.d.~sequence of examples is observed. An example can be selected as a center only immediately after it is observed, and it cannot be substituted later. The goal is to select a set of centers with a good $k$-median cost on the distribution which generated the sequence. We provide an efficient algorithm for this
  setting, and show that its multiplicative approximation factor is twice the
  approximation factor of an efficient offline algorithm. In addition, we show
  that if efficiency requirements are removed, there is an algorithm that
  can obtain the same approximation factor as the best offline algorithm.
  We demonstrate in experiments the performance of the efficient algorithm on real data sets. Our code is available at  \url{https://github.com/tomhess/No_Substitution_K_Median}.

\end{abstract}

\section{Introduction}

Clustering is an important unsupervised task used for various applications, including, for instance, anomaly detection \citep{leung2005unsupervised}, recommender systems \citep{shepitsen2008personalized} and image segmentation  \citep{ng2006medical}.  The $k$-median clustering objective is particularly useful when the partition must be defined using centers from the data, as in some types of image categorization \citep{dueck2007non} and video summarization \citep{hadi2006video}. 
While clustering has been classically applied to fixed offline data, in recent years clustering on sequential data has become a topic of ongoing research, motivated by various applications where data is observed sequentially, such as detecting
communities in social networks \citep{aggarwal2005online}, online recommender systems \citep{nasraoui2007performance} and online data summarization \citep{badanidiyuru2014streaming}.
Previous work on clustering sequential data \cite[e.g.,][]{guha2000clustering,ailon2009streaming, ackermann2012streamkm++} has typically focused on cases where the main limitation is memory; the clustering needs to be done on massive amounts of data, and so the data cannot be kept in memory in full. In this work, we study sequential $k$-median clustering in a new setting, which we call the \emph{no-substitution} setting. In this setting, an i.i.d.~sequence of examples is observed. An example can be selected as a center only immediately after it is observed, and it cannot be substituted later. The goal is to select a set of centers with a good $k$-median cost on the distribution which generated the sequence. This is a natural extension to clustering of the problem of irrevocable item selection from
a sequence, which is well-studied in various other settings (see, e.g., \citealt{kesselheim2017submodular,babaioff2007knapsack,babaioff2008online}).

The no-substitution setting captures applications of clustering in which the selection of each center involves an immediate and irrevocable action in the real world. For instance, consider selecting a small set of users from those arriving to a shopping website. These users will receive an expensive promotional gift,  where the goal is to select the users who will be the most effective in spreading the word about the product. Assuming a budget of $k$ gifts, this can be formalized as a $k$-median objective, with respect to a metric defined by connections between users, where the selected users are the centers. Offering the gift to a user  must be done immediately, before the user leaves the website. The gift also cannot later be reassigned to another user. This is captured by the no-substitution setting. As another example, consider selecting participants for a medical experiment from a stream of patients. The participants should represent the population, formalized as a $k$-median objective, and each participant should be selected before leaving the reception desk. These two examples demonstrate the usefulness of the no-substitution setting for real-life applications.

\paragraph{Our contributions.} We study the no-substitution setting in a general metric space, assuming that the data sequence is sampled i.i.d.~from an unknown distribution, and the goal is to minimize the distribution risk of the selected centers. The focus of
this work is obtaining theoretical guarantees for this setting, given a predefined length of stream, a fixed number of centers and a given confidence parameter.  We provide a computationally efficient and practical algorithm, called \algname, which uses as a black box a given clustering algorithm which is not restricted to the no-substitution setting. 
We show that the multiplicative approximation factor obtained by \algname\ is twice the factor obtained by the black-box algorithm, and that this factor of $2$ is tight. We further provide another algorithm, called \optalg, which obtains the same approximation factor as the best possible (though not necessarily efficient) offline algorithm. However, the computational complexity of \optalg\ is exponential in $k$. Whether there exists an efficient no-substitution algorithm with the same approximation factor as the best efficient offline algorithm, is an open question which we leave for future work. Lastly, we demonstrate \algname, the efficient algorithm, on real data sets.

\subsection*{Related Work} \label{sec:relatedwork}

%\todosivan{cite Michal}
We are not aware of previous works which study the no-substitutions setting for $k$-median clustering defined above.\footnote{Citation of a follow-up work by other authors, which cites an earlier unpublished version of this work, was removed for the anonymous submission.} Below we review previous work in related settings. \cite{ben2007framework} studied sample-based $k$-median clustering in the offline setting. In this setting, the entire set of sampled data points is observed, and then the $k$ centers are selected from this sample. For the case of a general metric space, \cite{ben2007framework} provides uniform finite-sample bounds on the convergence of the sample risk to the distribution risk of any choice of centers from the sample.

Algorithms studying clustering on sequential data have mainly assumed a fixed data set and an adversarial ordering, under bounded memory. In this setting, the approximation is with respect to the optimal clustering of the data set.
\cite{guha2000clustering} proposed the first single-pass constant approximation algorithm for the $k$-median objective with bounded memory. \cite{ailon2009streaming, chen2009coresets, ackermann2012streamkm++} develop algorithms for this setting using coreset constructions.
\cite{charikar2003better} design algorithms based on the facility-location objective, using a procedure proposed in \cite{meyerson2001online}, which also studies facility location under a random arrival order. \cite{braverman2016new} suggests a space-efficient technique to extend any sample-based offline coreset construction to the streaming (bounded-memory) model. \cite{lang2018online} considers the streaming $k$-median problem under a random arrival order.
Unlike the no-substitution setting, these algorithms can repeatedly change their selection of centers, or simply select a center that has appeared sometime in the past.

\cite{liberty2016algorithm} studies the online $k$-means objective with an
arbitrary arrival order, in a setting where each observed point must either be
allocated to an already-defined cluster or start a new cluster. This setting
can be seen as a variant of the no-substitution setting, since a chosen center
cannot be discarded later. However, the proposed algorithm selects $O(k\log m)$ centers, where $m$ is the sample size, and it is shown that in this adversarial setting, one must select more than $k$ elements to obtain a bounded approximation factor. \cite{lattanzi2017consistent} propose an online $k$-median algorithm which 
minimizes the number of necessary recalculations of a clustering.

The no-substitution setting bears a resemblance to the secretary problem under a cardinality constraint. In this setting, a set of limited cardinality must be selected with no substitutions from a sequence of objects, so as to optimize a given objective. \cite{bateni2010submodular,feldman2011improved,kesselheim2017submodular} study this setting when the objective is monotone and submodular. \cite{badanidiyuru2014streaming} suggest reformulating the $k$-median objective as a submodular function. However, this reformulation does not preserve the approximation ratio of the $k$-median objective. It also requires access to an oracle for function value calculations, which is not readily available in the sample-based sequential clustering setting. \cite{SabatoHe18} study a more general problem of converting an offline algorithm to a no-substitution algorithm in an interactive setting.

\section{Setting and Preliminaries} \label{sec:preliminaries}

For an integer $i$, denote $[i] := \{1,\ldots,i\}$.  Let $(\cX,\rho)$ be a
bounded metric space, and assume $\rho \leq 1$. For $c \in \cX$ and $r\geq 0$, let
$\ball(c,r) := \{ x \in \cX \mid \rho(c,x) \leq r \}$. Assume a probability distribution $P$ over $\cX$. Below, we assume $X \sim P$, unless explicitly noted otherwise. For $B \subseteq \cX$, denote $\P[B]:=\P[X \in B]$. 
A $k$-clustering is a
set of $k$ points $T = \{t_1,\ldots,t_k\} \subseteq \cX$ which represent the centers of the clusters. Given a probability
distribution $P$, the
$k$-median risk of $T$ on $P$ is
$R(P,T) := \E[\min_{i \in [k]}\rho(X,t_i)]$. For a finite set
$S \subseteq \cX$, $R(S,T)$ is the risk of $T$ on the uniform
distribution over $S$. We will generally assume an i.i.d.~sample $S \sim P^m$. For convenience of presentation, we treat $S$ as both a sequence and as a set interchangeably, ignoring the possibility of duplicate examples in the sample. These can be easily handled by using multisets, and taking the necessary precautions when selecting an element from $S$. When a minimization with respect to $\rho$ is performed, we assume that ties are broken arbitrarily.

Denote by $\OPT  \in \argmin_{T \in \cX^k}R(P,T)$ a specific optimal solution of the $k$-median clustering problem, where the minimization is over all possible $k$-clusterings in $\cX$; we assume for simplicity that such an optimizer always exists. Denote by $\OPT_S \in \argmin_{T \in \cX^k}R(S,T)$ a solution that minimizes the risk on $S$ using centers from $\cX$. 

In the no-substitution $k$-median setting, the algorithm does not know the
distribution $P$. It observes the i.i.d.~sample $S\sim P^m$ in a
sequence and selects centers from $S$. Formally, there are $m$ time steps. At time step $t$, a single example $x_t \sim P$  is observed and can be selected as a center. $x_t$ cannot be selected as a center at a later time step. Moreover, once a center is selected, it cannot be removed or substituted. The algorithm can select  $k$ elements from $S$ as centers, to form the $k$-clustering $T$.  
 The objective is to obtain a small $R(P,T)$, compared to the optimal $R(P,\OPT)$. 
 
 An \emph{offline $k$-median algorithm} $\cA$ takes as input a finite set of points $S$ from $\cX$ and outputs a $k$-clustering $T\subseteq S$. We say that $\cA$ is a \emph{$\beta$-approximation} offline $k$-median algorithm, for some $\beta \geq 1$, if for all input sets $S$, $R(S,\cA(S)) \leq \beta \cdot R(S, \OPT_S)$.
It is well known \cite[e.g.,][]{guha2000clustering} that for any data set $S$, $R(S, \argmin_{T \in S^k}R(S,T)) \leq  2R(S,OPT_S)$, and that this upper bound is tight.\footnote{The tightness can be observed by considering a star graph where the metric is the shortest path between vertices, and the center of the star is in $\cX \setminus S$.} Therefore, the lowest possible value for $\beta$ in a general metric space is $2$.

For a non-negative function $f(k,m,\delta)$, we denote by $O(f(k,m,\delta))$ a function which is upper-bounded by $C\cdot f(k,m,\delta)$ for some universal constant $C$, for any integer $k$, $\delta \in (0,1)$, and sufficiently large $m$.

\section{An Efficient Algorithm: \algname}
\label{sec:SKM}

\newcommand{\qpoint}{\mathrm{qp}}
\newcommand{\qball}{\mathrm{qball}}
\newcommand{\out}{\mathrm{out}}

\begin{algorithm}[t]
\caption{SKM}
\begin{algorithmic}[1]\label{SKM}
\REQUIRE $k, m\in \nats$, $\delta \in (0,1)$,  offline $k$-median algorithm $\mathcal{A}$, sequential access to \mbox{$S = (x_i)_{i=1}^m \sim P^m$}  
\ENSURE A $k$-clustering $T_{\out}\subseteq S$.
\STATE $q \leftarrow \frac{43\ln(\frac{2 m^2}{ \delta})}{m}$; $T_{\out} \leftarrow \emptyset$
\STATE Get $m/2$ samples from $S$; $S_1 \leftarrow (x_1,\ldots ,x_{m/2})$.
\STATE Run $\cA$ on $S_1$, and set $\{c_1,\ldots c_k \} \leftarrow \cA(S_1)$.
\FOR{$j=m/2+1 \text{ to } m$}
\STATE Get the next sample $x_{j}$
\IF{$\exists i \in [k]$ such that $ x_{j} \in \qball_{S_1}(c_i, q)$ and $T_{\out} \cap \qball_{S_1}(c_i,q)    = \emptyset$ \label{step:condition}}
\STATE $T_{\out} \leftarrow T_{\out} \cup \{ x_{j} \}$.
\ENDIF
\ENDFOR
\RETURN $T_{\out}$
\end{algorithmic}
\end{algorithm}

The first algorithm that we propose is called \algname\ (Sequential K-Median). 
\algname\ works in two phases. In the first phase, the incoming elements are observed and no element is selected. In the second phase, elements are selected based on the information gained in the first phase. 
\algname\ receives as input a confidence parameter $\delta$, the number of clusters $k$, the sequence size $m$, and access to a black-box offline $k$-median algorithm $\cA$.
The main challenge in designing \algname\ is to define a selection rule for elements from the second phase, based on the information gained in the first phase. This information should have uniform finite-sample convergence properties, so that the error of the solution can be bounded. In addition, the selection rule should guarantee selecting $k$ centers with a high probability.
\algname\ constructs this rule by combining the solution of $\cA$, calculated on the examples of the first phase, with estimations on the distribution.

\algname\ is listed in Alg.~\ref{SKM}. It uses the following notation. Denote the elements observed in the first phase by $S_1$, and those in the second phase by $S_2$. Elements from $S_2$ are selected as centers if they are close to the centers calculated by $\cA$ for $S_1$. Importantly, closeness is measured relative to the distribution of distances in $S_1$: An element is considered close to a center if its distance is smaller than all but at most a $q$ fraction of the points in $S_1$. 
Formally, for $x,y \in S$, define  $B(x,y) := \P[\ball(x,\rho(x,y))]$. This is the probability mass of points whose distance from $x$ is at most the distance of $y$. For a set of points $S$ and $x,y \in S$, let $\hat{B}_S(x,y)$ be the fraction of the points in $S\setminus \{x,y\}$ that are in $\ball(x,\rho(x,y))$. 
For  $q \in (0,1)$, let $\qpoint_S(x,q)$ be some point $y \in \argmin\{\rho(x,y) \mid y \in S, \hat{B}_S(x,y) \geq q\}$.  
Denote the ball in $\cX$ with center $x$ and radius determined by $\qpoint_S(x,q)$ by $\qball_S(x,q) := \ball(x,\rho(x,\qpoint_S(x,q)))$. 

The computational complexity of \algname\ is  $O(km\log(m))$ plus the complexity of the black-box algorithm $\cA$. In a memory-restricted online setting, a small variant of \algname\ can be used, which calculates the clustering of $\cA$ on half of $S_1$ and finds the ball radii based on $q$ using the second half of $S_1$, and a memory of $O(k\log(m/\delta))$ examples. Combined with a black-box $\cA$ which is itself online and memory-restricted, the result is an online memory-restricted no-substitution algorithm.

\subsection{Risk upper bound for \algname}

The following theorem provide the guarantee for \algname.
\begin{theorem} \label{thm:mtheorem}
  Suppose that \algname\ is run with inputs $k \in \nats$, $m \geq \max(2k,24)$, $\delta \in (0,1)$ and $\cA$, where $\cA$ is a $\beta$-approximation offline $k$-median algorithm. 
  For any $\gamma\in (0,\frac{1}{2})$ and any distribution $P$ over $\cX$, with a probability at least $1-\delta$, 
\begin{align*}
  R(P,T_{\out}) &\leq   (2+2\gamma) \beta R(P,\OPT) \\&+   ((2+2\gamma)\beta+1)\sqrt{(k\ln(m/2) + \ln(4/\delta))/m} \\
  &+ (1/\gamma)\cdot (44k\log(2m^2/\delta))/m.
\end{align*}

\end{theorem}

\thmref{mtheorem} gives a range of trade-offs between additive and multiplicative errors, depending on the value of $\gamma$. In particular, by setting $\gamma = \sqrt{(k \log(m/\delta))/m}$ and noting that $R(P,\OPT) \leq 1$, we get 
\begin{align}\label{eq:mtheorem}
  R(P,T_{\out}) &\leq 2\beta R(P,\OPT) %\\
%                &\qquad
                  +\beta\cdot O(\sqrt{\frac{k\log(\frac{m}{\delta})}{m} }).%\notag
\end{align}

This guarantee can be compared to the guarantee of an offline algorithm that uses the same $k$-median algorithm $\cA$ as a black box. As shown in \cite{ben2007framework}, for $S \sim P^m$, with a probability at least $1 -\delta$, for every $k$-clustering $T \subseteq S$ and for $T=\OPT$,
\begin{align}\label{eq:bendavid}
  &|R(P,T)-R(S,T)|
%  \\
%  &
    \leq O(\sqrt{\frac{k\ln m + \ln(\frac{1}{\delta})}{m}}).%\notag
\end{align}
Denote the RHS by $O(f(m,k,\delta))$. Therefore,
  \begin{align*}
    R(S,\mathcal{A}(S)) &\leq \beta R(S,\OPT_S) \leq \beta R(S,\OPT)\\
    &\leq \beta R(P,\OPT) +\beta\cdot O(f(m,k,\delta)).
  \end{align*}
  
Since \eqref{bendavid} holds also for $T = \cA(S)$, it follows that $R(P,\cA(S)) \leq \beta R(P,\OPT) + \beta \cdot O(f(m,k,\delta)).$ Therefore,  the additive errors of this guarantee and that of \thmref{mtheorem} have a similar dependence on $m,k$, $\delta$ and $\beta$. When $m \rightarrow \infty$, the additive errors go to zero, and there remains the approximation factor of $2\beta$ for \algname, instead of $\beta$ for the offline algorithm. We show in \secref{lowerbound} that the $2\beta$ approximation factor is tight.

To prove \thmref{mtheorem}, we first prove that with a high probability, \algname\ succeeds in selecting $k$ centers from $S_2$. This requires showing that the estimate of the mass of $\qball_{S_1}(c_i,q)$ using $S_1$ is close to its true mass on the distribution. 
We use the following lemma, proved in the supplementary material using the empirical Bernstein's inequality of \cite{maurer2009empirical}:
\begin{lemma}\label{lem:empbernstein}
  Let $Y_1,\ldots,Y_n$ be i.i.d.~random variables over $[0,1]$ with mean $\mu$. Let $\hat{\mu} = \frac{1}{n}\sum_{i\in[n]}Y_i$ be their empirical mean. Then, with a probability at least $1-\delta$, $\hat{\mu} \leq \max(16\ln(\frac{2}{\delta})/(n-1),2\mu)$.
\end{lemma}

This result is used in the proof of the following lemma. For readability, we denote the sizes of $S_1$ and $S_2$ by $m_1,m_2$ respectively.
  \begin{lemma} \label{lem:minimum_q}
    For every distribution $P$ over $\cX$, if $m_1 \geq \max(k,12)$ then with a probability at least $1-\delta/2$, for every $i \in [k]$, \algname\ selects a point in $\qball_{S_1}(c_i, q)$ from $S_2$.
\end{lemma} 

\begin{proof}
  For $x,y \in S_1$, denote $\hat{B} := \hat{B}_{S_1}(x,y)$. Apply \lemref{empbernstein} by letting $Y_1,\ldots,Y_n$ stand for the indicators $\one[z \in B(x,y)]$ for $z \in S_1 \setminus \{x,y\}$, $n = m_1-2$, $\hat{\mu} = \hat{B}, \mu = B(x,y)$. It follows that with a probability at least $1-\delta$,
  if $\hat{B} \geq 16\ln(\frac{2}{\delta})/(m_1-3)$, then $B(x,y) \geq \hat{B}/2$, hence $B(x,y) \geq  8\ln(\frac{2}{\delta})/(m_1-3)$. 
  By a union bound on the pairs in $S_1$, we have that with a probability of $1-\delta/4$, for all pairs $x,y \in S_1$, 
\begin{align*}
\hat{B}_{S_1}(x,y) &\geq 16\ln(\frac{8m_1^2}{\delta})/(m_1-3) \\ \Longrightarrow \quad
    B(x,y) &\geq \hat{B}_{S_1}(x,y)/2.
\end{align*}
  In particular, this holds for $x = c_i$ and $y = y_i := \qpoint_{S_1}(c_i,q)$,
  where $c_1,\ldots,c_k$ are the centers returned by $\cA$ in \algname.
  Denote $\hat{B}_i = \hat{B}_{S_1}(c_i,y_i)$. By definition of $y_i$, for all $i \in [k]$, $\hat{B}_i \geq q$. In addition, by definition of $B(\cdot,\cdot)$ and $\qball$, we have that $\P[\qball_{S_1}(c_i,q)] = B(c_i,y_i)$.
  Since $m_1 \geq 12$, we have $m_1 -3 \geq 3m_1/4$. Therefore,
$\hat{B}_i \geq q = 43\ln(2m^2/\delta)/m \geq 16 \ln(8m_1^2/\delta)/(m_1 - 3)$. 
%$\hat{B}_i \geq q = 64\ln(4m/\delta)/m \geq 16 \ln(8m_1^2/\delta)/(m_1 - 3)$. 
Therefore, with a probability at least $1-\delta/4$, $S_1$ satisfies that for all $i \in [k], \P[\qball_{S_1}(c_i,q)] \geq q/2 \geq \frac{\ln(\frac{4 k}{ \delta})}{m_2} =: \eta$, 
    where we used $m_1 \geq k$. 
      If this event holds for $S_1$, then the probability over $S_2 \sim P^{m_2}$ that $S_2 \cap \qball_{S_1}(c_i,q) = \emptyset$ is at most $(1-\eta)^{m_2}\leq \exp(-m_2\eta).$ By a union bound, the probability that for some $c_i$ a center is not found in $S_2$ is at most $k\exp(-m_2\eta) \leq \delta/4$. Combining the two events, we conclude that the probability that a point is found in $S_2$ for all centers is at least $1-\delta/2$.
  \end{proof}

  We now bound the risk of the output of \algname, under the assumption that indeed all centers have been successfully selected.
  The condition in step~\ref{step:condition} of the algorithm guarantees that all the selected centers are  in the $\qball$ around the centers returned by $\cA$. The following two lemmas bound the risk that the selected centers induce compared to the original centers. The lemmas are formulated more generally to apply to a general distribution. The first lemma considers a single center. For a distribution $Q$ over $\cX$ and $c,t \in \cX$, denote $B^o_Q(c,t) := \P_{X \sim Q}[\rho(X,c) < \rho(t,c)]$.

\begin{lemma}  \label{lem:quantilemarkov}
  Let $\tau \in (0,1)$. Let $Q$ be a distribution over $\cX$. Let $c\in \cX$,  $t \in \cX$ such that $B^o_Q(c,t) \leq \tau$. Then
 %  $r \geq 0$ be a radius such that % $\forall r' > r, \P[\ball(c, r')] \geq \tau$, and $\forall r' < r, \P[\ball(c, r')] < \tau$.
%  $\P[\rho(c,X) < r] < \tau$, and let $t \in \ball(c,r)$. Then %r = \inf\{r' \geq 0 \mid \P[\ball(c, r')] \geq \tau\}$ and $t \in \ball(c,r)$. Then:
%\[
$R(Q,\{t\}) \leq (1+1/(1-\tau))R(Q,\{c\}).$
%\] 
\end{lemma}

\begin{proof}
Denote $r := \rho(t,c)$. Using the triangle inequality, and letting $X \sim Q$, we have 
\begin{align*}
  R(Q,\{t\}) &=\E[\rho(X,t)]\leq \E[\rho(X,c)+\rho(t,c)] \\
  &= R(Q,\{c\})+r.
\end{align*}
To upper-bound $r$, note that by the conditions on $t$, $\P[ \rho(X,c) \geq r] \geq 1-\tau$. Therefore,
  $R(Q,\{c\}) \geq  r \cdot \P[ \rho(X,c) \geq r]  \geq (1-\tau)r.$
It follows that  $r  \leq  R(Q,\{c\})/(1-\tau)$, which completes the proof.
\end{proof} 
The lemma above provides a multiplicative upper bound on the risk obtained when replacing a center $c_i$ with another center $t_i$. However, this upper bound is only useful if $\tau$ is small. In the general case, an additive error term cannot be avoided. For instance, suppose that the optimal clustering has a risk of zero, and there is at least one very small cluster. In this case, the algorithm might not succeed in choosing a good center for this cluster, and some additive error will ensue. The following lemma bounds the overall risk of the clustering when all centers are replaced.

\begin{lemma}\label{lem:concentrate}
   Let $\tau \in (0,1)$ and let $Q$ be a distribution over $\cX$. Let $O = \{c_1,\ldots,c_k\}\subseteq \cX$, and $T = \{t_1,\ldots,t_k\} \subseteq \cX$ such that $B^o_Q(c_i,t_i) \leq \tau$. Then for any $\gamma \in (0,\half)$,
\[
R(Q,T) \leq (2+2\gamma)R(Q,O) + k\tau/\gamma.
\] 
\end{lemma}
\begin{proof}
  Let $C_i := \{ x \in \cX \mid i = \argmin_{j \in [k]} \rho(c_j,x)\}$ and $\beta_i := \P_{X \sim Q}[X \in C_i]$.
  Let $q_i :=  B^o_Q(c_i,t_i)/\beta_i$, and let $Q_i$ be the conditional distribution of $X \sim Q$ given $X \in C_i$. 
   Distinguish between two types of clusters. If $q_i \geq \gamma$, then $\gamma \leq q_i \leq \tau/\beta_i$, where the second inequality follows from the assumption on $t_i$. Thus $\beta_i \leq \tau/\gamma$. Since $\rho \leq 1$, $R(Q_i,\{t_i\}) \leq 1$.  Therefore,  \mbox{$\sum_{i: q_i \geq \gamma}\beta_i \cdot R(Q_i, \{t_i\}) \leq k\tau/\gamma$.}
    On the other hand, if $q_i < \gamma$, then
    \begin{align*}
      B_{Q_i}^o(c_i,t_i) &= \P_{X \sim Q}[\rho(X,c_i) < \rho(t_i,c_i) \mid X \in C_i] \\ 
      &\leq B_{Q}^o(c_i,t_i)/\beta_i = q_i < \gamma.    
    \end{align*}
      Thus, \lemref{quantilemarkov} holds for $\tau := \gamma$, $Q:= Q_i$,  $t := t_i$ and $c := c_i$, hence $R(Q_i,\{t_i\}) \leq (1+\frac{1}{1-\gamma})R(Q_i,\{c_i\}).$ Since $\gamma \in (0,\half)$, we have $1+\frac{1}{1-\gamma} \leq 2 + 2\gamma$.
      Therefore,
      \begin{align*}
        \sum_{i:q_i < \gamma} \beta_i \cdot R(Q_i,\{t_i\}) &\leq (2+2\gamma)\sum_{i:q_i < \gamma} \beta_i \cdot R(Q_i,\{c_i\})\\
        &\leq (2+2\gamma)\cdot R(Q,O).
      \end{align*}
We thus have
  \begin{align*}
    R(Q,T) &\leq \sum_{i \in [k]} \beta_i \cdot R(Q_i,\{t_i\}) \\&= \sum_{i:q_i < \gamma} \beta_i \cdot R(Q_i,\{t_i\}) + \sum_{i:q_i \geq \gamma} \beta_i \cdot R(Q_i,\{t_i\}) \\
    &\leq (2+2\gamma)\cdot R(Q,O) + k\tau/\gamma,
  \end{align*}
  which completes the proof.
\end{proof}

Using the results above, \thmref{mtheorem} can now be proved.

\begin{proof}[Proof of \thmref{mtheorem}.]
  Recall that $S_1,S_2$ are independent i.i.d.~samples of size $m_1,m_2$ drawn from $P$. By Hoeffding's inequality and the fact that $\rho \leq 1$ we have that for any fixed $k$-clustering $T$,
%\[
$\P[|R(P,T)-R(S_1,T)|\geq \epsilon] \leq 2e^{-2\epsilon^2 m_1}.$ 
%\]
By a union bound on all the $k$-clusterings in $S_2$ and on $T=\OPT$, we get that with a probability $1-\delta/2$, all such clusterings $T$ satisfy
\begin{align}
  &|R(P,T)-R(S_1,T)| \leq \sqrt{(k\ln(m_2) + \ln(4/\delta))/(2m_1)} \notag \\
  &\qquad= \sqrt{(k\ln(m/2) + \ln(4/\delta))/m} =: \epsilon_1\ \label{eq:epsilon2},
\end{align}
where we used $m_1 = m_2 = m/2$. 

In addition, by \lemref{minimum_q}, with a probability at least $1-\delta/2$, \algname\ selects $k$ centers from $S_2$. The two events thus hold simultaneously with a probability at least $1-\delta$. Condition below on these events and let $t_1,\ldots,t_k$ be the selected centers, ordered so that $t_i \in \qball_{S_1}(c_i,q)$. Denote $N_i = | \{ z \in S_1 \mid \rho(c_i,z) < \rho(c_i,t_i)\}|$. 
Since $t_i \in \qball_{S_1}(c_i, q)$, we have by definition of $\qball$ that
$N_i/|S_i| \leq ((m_1-2)q+1)/m_1 \leq q + 1/m_1$.
Therefore, \lemref{concentrate} holds with $Q$ set to the uniform distribution on $S_1$, $O := \cA(S_1)$, and $\tau := q + 1/m_1$. Hence,
\[
  R(S_1,T_{\out}) \leq (2+2\gamma)R(S_1,\mathcal{A}(S_1)) + k(q + 1/m_1)/\gamma.
  \]
  By the assumptions on $\cA$ and by \eqref{epsilon2},
  \begin{align*}
    R(S_1,\cA(S_1)) &\leq \beta R(S_1,\OPT_{S_1}) \leq \beta R(S_1,\OPT) \\
    &\leq \beta (R(P,\OPT) + \epsilon_1). 
  \end{align*}
  In addition, $R(P,T_{\out}) \leq R(S_1,T_{\out})+\epsilon_1.$ Combining the inequalities and noting that $m_1 = m/2$, we get 
\begin{align*}
  R(P,T_{\out}) &\leq (2+2\gamma)\beta (R(P,\OPT) + \epsilon_1)\\
 &\qquad+ k(q + 2/m)/\gamma + \epsilon_1.
\end{align*}
  The theorem follows by setting $q$ as in \algname. 
\end{proof}

We have thus shown that \algname\ obtains an approximation factor at most twice that of the offline algorithm. In the next section, we show that this upper bound on the multiplicative factor is tight.

%%% Local Variables:
%%% mode: latex
%%% TeX-master: "No_sub_k_median_AISTAS2020"
%%% End:

\subsection{Tightness of the multiplicative factor } \label{sec:lowerbound}

In this section we show that the multiplicative approximation factor of $2\beta$ given in  \eqref{mtheorem} is tight for \algname.

Let $\mathcal{A}$ be an offline $k$-median algorithm, which for every sample $S$ returns a $k$-clustering $T \subseteq S$ that minimizes $R(S,T)$.
As discussed in \secref{preliminaries}, $\mathcal{A}$ is a  $2$-approximation offline $k$-median algorithm. Thus, $\beta = 2$ in \eqref{mtheorem}.
We now show that \algname\ in this case cannot have a multiplicative factor of less than $4$, thus showing that the approximation factor is tight. Moreover, this holds for any setting of $q$, not necessarily the one used in Alg.~\ref{SKM}.
Note that if the probability mass of the $q$-ball set by \algname\ is smaller than $\log(1/\delta)/m$, then the probability of finding a center in the second phase is less than $1-\delta$. Therefore, one must have $q \geq \log(1/\delta)/m$. In addition, one must have $q \equiv q(m) \rightarrow 0$ when $m \rightarrow \infty$, otherwise the additive error would not vanish for large $m$.\footnote{To see this, consider a case with a very small optimal risk, in which one of the clusters has a probability mass of $q/2$. With a constant probability, the center for this cluster will be selected from another cluster, resulting in an additive error of $\Omega(q)$.} The following theorem shows that for any $q$ which satisfies these requirements, the approximation factor of \algname\ is at least $4$. The proof of the theorem is provided in the supplementary material.

\begin{theorem} 
  \label{thm:2example}
  Consider running \algname\ with any setting of $q=q(m)$  such that $q(m) \rightarrow 0$ when $m \rightarrow \infty$, and $q(m)\geq \log(1/\delta)/m$ for all $m$. Then, the multiplicative factor  of \algname\ cannot be smaller than $4 = 2\beta$ for $\cA$ as defined above.
\end{theorem}

We conclude that the multiplicative factor of $2\beta$ for \algname\ is tight. \algname\ uses a black-box algorithm $\cA$, and it is computationally efficient if $\cA$ is computationally efficient. In the next section, we show that if efficiency limitations are removed, there is an algorithm for the no-substitution setting that obtains the same approximation factor as an optimal (possibly also inefficient) offline algorithm.

\section{Obtaining the Optimal Approximation Factor: \optalg} \label{sec:sectionSKM2}

If efficiency considerations are ignored, the offline algorithm can use a $\beta$-approximation algorithm with the best possible $\beta$, which is equal to $2$, as discussed above. Using \eqref{bendavid}, this gives the following guarantee for the offline algorithm:
\begin{align*}
  &R(P,\argmin_{T \in S^k}R(S,T)) \\
  &\quad\leq 2R(S, \OPT_S) + O\left(\sqrt{(k\log(m) +\log(1/\delta))/m}\right).
\end{align*}

We now give an algorithm for the no-substitution setting, which obtains the same approximation factor of $2$, and a similar additive error to that of the offline algorithm.
The algorithm, called \optalg, is listed in \algref{optalg}. It receives as input the confidence parameter $\delta$, the number of clusters $k$, and the sequence size $m$. Similarly to \algname, it also works in two phases, where the first phase is used for estimation, and the second phase is used for selecting centers. The first phase is further split to sub-sequences $S_0,S_1,\ldots,S_k$. The second phase is denoted $\bar{S}$.

The main challenge in designing \optalg\ is to make sure that elements are selected as centers only if it will later be possible, with a high probability, to select additional centers so that the final risk will be near-optimal. To this end, we define a recursive notion of \emph{goodness}. For a set of size $k$, we say that it is good if its risk on $S_0$ is lower than some threshold. For a set of size less than $k$, it is good if there is a sufficient probability to find another element to add to this set, such that the augmented set is good. The following definition formalizes this.

  \newcommand{\estgood}{\text{good}}

  \begin{definition}
    Let $Z \subseteq \cX$ of size at most $k$. Let $\rsk > 0$ and $q \in (0,1)$. The predicate $(\rsk,q)$-\emph{\estgood} is defined as follows, with respect to the sub-samples $S_0,S_1,\ldots,S_k \subseteq \cX$.
  \begin{itemize}
    \item For $Z$ of size $k$, $Z$ is $(\rsk,q)$-\estgood\ (or simply $\rsk$-\estgood) if $R(S_0,Z) \leq \rsk$.  
    \item For $Z$ of size $j \in \{0,\ldots,k-1\}$, define $\hat{\psi}_{\rsk,q}(Z) := \P_{X \sim S_{j+1}}[Z\cup\{X\} \text{ is $(\rsk,q)$-\estgood}]$. $Z$ is $(\rsk,q)$-\estgood\ if $\hat{\psi}_{\rsk,q}(Z) \geq 2q$.

  \end{itemize}
\end{definition}

  \optalg\ sets the value of $q$ depending on the input parameters, and finds a value for $\rsk$ such that  $\emptyset$ is $(\rsk,q)$-good. It then iteratively gets the examples, and adds the observed example as a center if the addition preserves the goodness of the solution collected so far.  We show below that if $\emptyset$ is $(\rsk,q)$-\estgood\ for $q$ as defined in \algref{optalg}, then with a high probability \optalg\ will succeed in selecting $k$ centers with a risk at most $\rsk$ on $S_0$, and that this will result in a near-optimal $k$-clustering. 
 \optalg\ has a computational complexity exponential in $k$, since it considers recursively all the elements of $S_1 \times  S_2 \times  \ldots \times  S_k$ .
We prove the following result for \optalg.

\begin{algorithm}[t]
\caption{\optalg}
\begin{algorithmic}[1]\label{alg:optalg}
\REQUIRE  $k, m\in \nats$, $\delta \in (0,1)$, sequential access to $S = (x_1,\ldots,x_m) \sim P^m$.
\ENSURE A $k$-clustering $T_{\out}\subseteq S$
\STATE \label{lineqdefinition}$q \leftarrow (32k^2\log(8m) + 32k\log(8/\delta))/m$; $T_{\out} \leftarrow \emptyset$  
\STATE Get $m/2$ examples from $S$. Set $m/4$ examples as $S_0$, and split the rest of the examples equally between $S_1,\ldots,S_k$. 
\STATE \label{linerR}
Set $\beta_m := 1/\sqrt{m}$. Let $\rsk \leftarrow \min \{ r  = \beta_m(1+\beta_m)^n \mid n \in \nats, \text{ and }\emptyset\text{ is $(\rsk,q)$-\estgood}\}$.
\FOR{$j=m/2+1 \text{ to } m$}
\STATE Get the next sample $x_{j}$.
\STATE \textbf{If} $|T_{\out}| < k$ and $T_{\out} \cup \{x_j\}$ is $(\rsk,q)$-\estgood\ \textbf{then} $T_{\out} \leftarrow T_{\out} \cup \{x_{j}\}$. %m/2 was m_1
\ENDFOR
\RETURN $T_{\out}$
\end{algorithmic}
\end{algorithm}

\begin{theorem}
  \label{thm:rtheorem}
  Suppose that \optalg\ is run with inputs $k,m \in \nats$ and $\delta \in (0,1)$. For any $\gamma\in (0,\frac{1}{2})$ and distribution $P$ over $\cX$, with a probability at least $1-\delta$,
  \begin{align*}
R(P,T_{\out}) \leq\, &(2+2\gamma)R(P,\OPT) \\
&+\frac{1}{\gamma}\cdot O\left((k^3\log(m)+k^2 \log(1/\delta))/m\right) \\
&+ O\left((k\log(m) + \log(1/\delta))/m\right).
  \end{align*}
\end{theorem}
By setting $\gamma = \sqrt{(k^3\log(m)+k^2 \log(1/\delta))/m}$ and noting the $R(P,\OPT) \leq 1$, we get 
\begin{align*}
  &R(P,T_{\out}) \\
                  & \leq 2 R(P,\OPT) +O\left(\sqrt{(k^3\log(m)+ k^2\log(1/\delta))/m }\right).
\end{align*}
As discussed above, this is the same multiplicative approximation factor as the optimal offline algorithm. The additive error is larger by a factor of $k$.

We now prove \thmref{rtheorem}. 
Note that by the definition of goodness for $Z$ of size $k$, it follows that if \optalg\ succeeds in selecting $k$ centers, then the solution it finds has a risk of at most $\rsk$ on $S_0$. We thus need to show that indeed $k$ centers are selected with a high probability, that $\rsk$ is close to the optimal achievable risk, and that the risk on $S_0$ is close to the risk on $P$. 
We use the following lemma, proved in the supplementary material based on Bernstein's inequality.
\begin{lemma}\label{lem:bernstein}
  Let $Y_1,\ldots,Y_n$ be i.i.d.~random variables in $[0,1]$ with mean $\mu \geq 10\ln(\frac{1}{\delta})/n$. Let $\hat{\mu} = \frac{1}{n}\sum_{i\in[n]}Y_i$ be the empirical mean. Then, with a probability at least $1-\delta$, $\hat{\mu} \geq \mu/2$. 
\end{lemma}

Denote the sizes of $S_0,S_1,\ldots,S_k,\bar{S}$ by $m_0,m_1,\ldots,m_k,\bar{m}$ respectively.
First, we show that \optalg\ selects $k$ centers with a high probability.

\begin{lemma}
  \label{lem:rselectk}
  With a probability at least $1-\delta/2$, by the end of the run \optalg\ has collected $k$ centers.
\end{lemma}

\begin{proof}
  Let $\Upsilon = \{\beta_m(1+\beta_m)^n \mid n \in \nats\} \cap (0,1)$ be the possible values of $\rsk$ examined by the algorithm which are smaller than $1$.
  Note that since $\exp(x/2) \leq 1+x$ for $x \in (0,1)$, the largest $n$ such that $\beta_m(1+\beta_m)^n < 1$ satisfies $\beta_m \exp(n\beta_m/2) < 1$. Therefore, $|\Upsilon| \leq \frac{2}{\beta_m}\log(1/\beta_m) = \sqrt{m}\log(m) \leq m$.
  By \lemref{empbernstein} and a union bound, with a probability at least $1-\delta/4$, for any $\rsk \in \Upsilon$, $j \in \{0,\ldots,k-1\}$, and $T \subseteq \bar{S}$ of size $j$,
  $\hat{\psi}_{\rsk,q}(T) %&
  \geq 16 \ln(8\bar{m}^k/\delta)/(m_{j+1}-1) % \quad \\
    %&
    \Rightarrow \quad\P[T \cup \{X\}\text{ is $(\rsk,q)$-\estgood}] \geq \hat{\psi}_{\rsk,q}(T)/2.$
  Condition below on this event. 
      Let $\rsk$ be the value selected by \optalg, let $T_i$ be the set of points collected by the algorithm until iteration $i$, and let $j = |T_i| < k$.
    If $T_i = \emptyset$, then it is $(\rsk,q)$-\estgood\ by the definition of $\rsk$. Otherwise, it is $(\rsk,q)$-\estgood\ by the condition on line \ref{step:condition}. 
    Therefore, by definition, $\hat{\psi}_{\rsk,q}(T_i) \geq 2q$. This implies the LHS of the implication above, hence \mbox{$\P[T_i\cup \{X\}\text{ is $(\rsk,q)$-\estgood}] \geq q$}.

    Therefore, conditioned on the event above, the probability that the next sample $x_j$ satisfies that $T_i \cup \{x_j\}$ is $(\rsk,q)$-\estgood\ is at least $q$. Since this holds for all iterations until there are $k$ centers in $T_i$, the probability that the algorithm collects less than $k$ centers is at most the probability of obtaining less than $k$ successes in $\bar{m} = m/2$ independent experiments with a probability of success $q$.
    Let $\hat{s}$ be the empirical fraction of successes on $m/2$ experiments. By \lemref{bernstein}, since $q \geq 10\log(4/\delta)/(m/2)$, with a probability $1-\delta/4$, $\hat{s} \geq q/2$. Since $q \geq 2k/m$, we have $\hat{s} \geq k/m$.  Therefore, taking a union bound, with a probability of at least $1-\delta/2$, the algorithm selects $k$ centers.    
\end{proof}

We now show that the value of $\rsk$ selected by \optalg\ is close to the optimal risk.
By Hoeffdings's inequality and a union bound over the possible choices of $T$, for all $T \subseteq S \setminus S_0$ of size $k$, with a probability $1-\delta/4$,
$|R(P,T)-R(S_0,T)| \leq \sqrt{(2k\ln (m) +2\ln(\frac{8}{\delta}))/m}$. 
Call this event $E_0$ and denote the RHS by $\epsilon_2$. 

\begin{lemma}
\label{lem:Rbound}
Let $\gamma \in (0,\frac{1}{2})$, and define the value $\rsk_0 := ((2+2\gamma)R(P,\OPT) +4qk/\gamma+ \epsilon_2)$. With a probability of $1-\delta/4$, $E_0$ implies that the value of $\rsk$ set by \optalg\ satisfies  $\rsk \leq (1+\beta_m)\rsk_0$.
\end{lemma}
\begin{proof} 
  Let $j \in \{0,\ldots,k\}$. For sets $D_1,\ldots,D_j$, denote by $\bar{D}_j$ the collection of all sets of size $j$ that include exactly one element from each of $D_1,\ldots,D_j$. 
We start by showing that with a high probability, there exist sets $D_1,\ldots,D_k$ such that for all $i\in [k]$, $D_i \subseteq S_i$, $|D_i|\geq 2qm_i$, and $\max_{Z \in \bar{D}_k}R(S_0,Z) \leq \rsk_0.$ Let $\OPT = \{o_1,\ldots,o_k\} \subseteq \cX$ be an optimal $k$-clustering for $P$. 
For $i \in [k]$, let $\alpha_i \geq 0$ such that $\P[\ball(o_i,\alpha_i)] \geq 4q$ and $\P[\rho(X,o_i) < \alpha_i] \leq 4q$. Let $D_i=\ball(o_i, \alpha_i) \cap S_i$. Denote $B_i = \P[\ball(o_i, \alpha_i)]$. By \lemref{bernstein}, since
$B_i \geq 4q \geq 10\ln(4k/\delta)/m_i$, we have that with a probability at least $1-\delta/4$, for all $i \in [k]$, $|D_i|/|S_i| \geq 2q$, as required.

We now show that  $\max_{Z \in \bar{D}_k} R(S_0,Z) \leq \rsk_0$.
By the definition of $\alpha_i$, for any $d_i \in D_i$ we have $B^o_{P}(c_i,d_i) \leq 4q$, where $B^o$ is defined above \lemref{quantilemarkov}. Therefore, the conditions of \lemref{concentrate} hold with $Q := P$, $O := \OPT$, $T := Z$ and $\tau := 4q$. Hence, for $\gamma \in (0,\half)$, 
\[
  R(P, Z) \leq (2+2\gamma)R(P,\OPT) + 4qk/\gamma.
\]
Under $E_0$, we get that for all $Z \in \bar{D}_k$, $R(S_0, Z) \leq (2+2\gamma)R(P,\OPT) + 4qk/\gamma + \epsilon_2 \equiv \rsk_0.$

Lastly, we show that the existence of $D_1,\ldots,D_k$ implies an upper bound on the value of $\rsk$ set by the algorithm. First, we show that $\emptyset$ is $(\rsk_0,q)$-\estgood. This can be seen by induction on the definition of goodness: For $|Z| = k$, all $Z \in \bar{D}_k$ are $(\rsk_0,q)$-\estgood\ since $R(S_0, Z) \leq \rsk_0$. Now, suppose that all sets $Z \in \bar{D}_{j}$ for some $j \in [k]$ are $(\rsk_0,q)$-good, and let $Z' \in \bar{D}_{j-1}$. Then, since for all $x\in D_{j}$ we have $Z' \cup \{x\} \in \bar{D}_{j}$, it follows that
$\hat{\psi}_{\rsk_0,q}(Z') %&
= \P_{X \sim S_j}[Z' \cup \{X\} \text{ is $(\rsk_0,q)$-\estgood}] %\\
  %&
  \geq |D_j|/|S_j| \geq 2q.$
  Therefore, by definition, $Z'$ is $(\rsk_0,q)$-\estgood. By induction, we conclude that $\emptyset \in \bar{D}_0$ is also $(\rsk_0,q)$-\estgood. Clearly, $\emptyset$ is also $(\rsk_1,q)$-\estgood\ for any $\rsk_1 \geq \rsk_0$. 
  Since the value $\rsk$ selected by \optalg\ is set to the smallest value $\beta_m(1+\beta_m)^n$ such that $n$ is natural and $\emptyset$ is $(\rsk,q)$-good, and since $\rsk_0 \geq \epsilon_2 \geq \beta_m$, we conclude that $\rsk \leq \rsk_0(1+\beta_m)$, as required.
\end{proof}

The proof of \thmref{rtheorem} can now be provided.
\begin{proof}[Proof of \thmref{rtheorem}]
  Assume that $E_0$ holds, as well as the events of \lemref{rselectk} and \lemref{Rbound}. This occurs with a probability at least $1-\delta$. By \lemref{rselectk} the algorithm selects $T_{\out}$ which is of size $k$ and is $(\rsk,q)$-good.  Thus, by the definition of goodness, $R(S_0,T_{\out}) \leq \rsk$. By $E_0$, $R(P,T_{\out}) \leq \rsk + \epsilon_2$. By \lemref{Rbound}, 
$\rsk \leq (1+\beta_m)((2+2\gamma)R(P,\OPT) +4qk/\gamma+ \epsilon_2).$
 The theorem follows by plugging in the values of $\beta_m,q,\epsilon_2$ and simplifying. 
\end{proof}

\section{Experiments} \label{sec:experiments}

We demonstrate \algname\ \footnote{ Our code is available
at  \url{https://github.com/tomhess/No_Substitution_K_Median}.} on 3 datasets: \texttt{MNIST}
\citep{lecun1998gradient}, \texttt{Covertype}\footnote{Reuse of this database is unlimited with retention of copyright notice for Jock A. Blackard and Colorado State University.}  and \texttt{Census 1990} \citep{Dua2019}.  While Alg.~\ref{SKM}
uses $q =43\ln(2 m^2/\delta)/m$, it can be seen from the
proof of \thmref{mtheorem} that except for very small values of $m$, the guarantees of \algname\ hold also with significantly smaller values. In the experiments we used $q = 9\ln(2m^2/\delta)/m$. In all experiments,
the features were normalized, and PCA was used to reduce the dimension, so that 95\% of the signal was retained. As black-box $k$-median
algorithms, we used the implementation of $k$-medoids in \cite{Novikov2019}, and the BIRCH algorithm \cite{zhang1997birch}, implemented in \cite{scikit-learn}. All risks were estimated on
the same holdout set, and averaged over 20 runs. \figref{exp} reports the ratio between the clustering risk obtained by
\algname\ and the risk of the offline algorithm, for $k$-medoids. Results for BIRCH are reported in the supplementary material. It can be seen that in practice, the risk ratio obtained by \algname\ is usually close to $1$. The results for large stream sizes, provided in the supplementary material, show a convergence to values very close to $1$. As expected, the convergence is slower for larger values of $k$.

\begin{figure}[t]
  \begin{center}
  \includegraphics[width = 0.49\columnwidth]{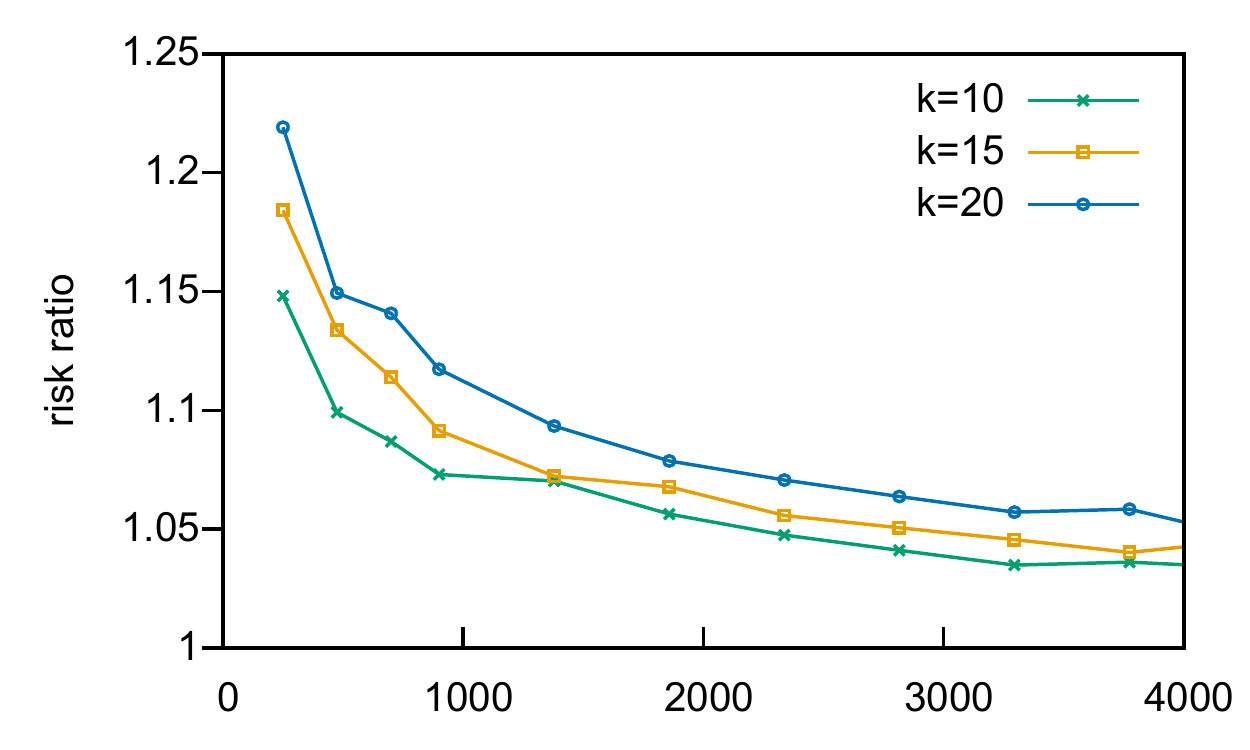}
  \includegraphics[width = 0.49\columnwidth]{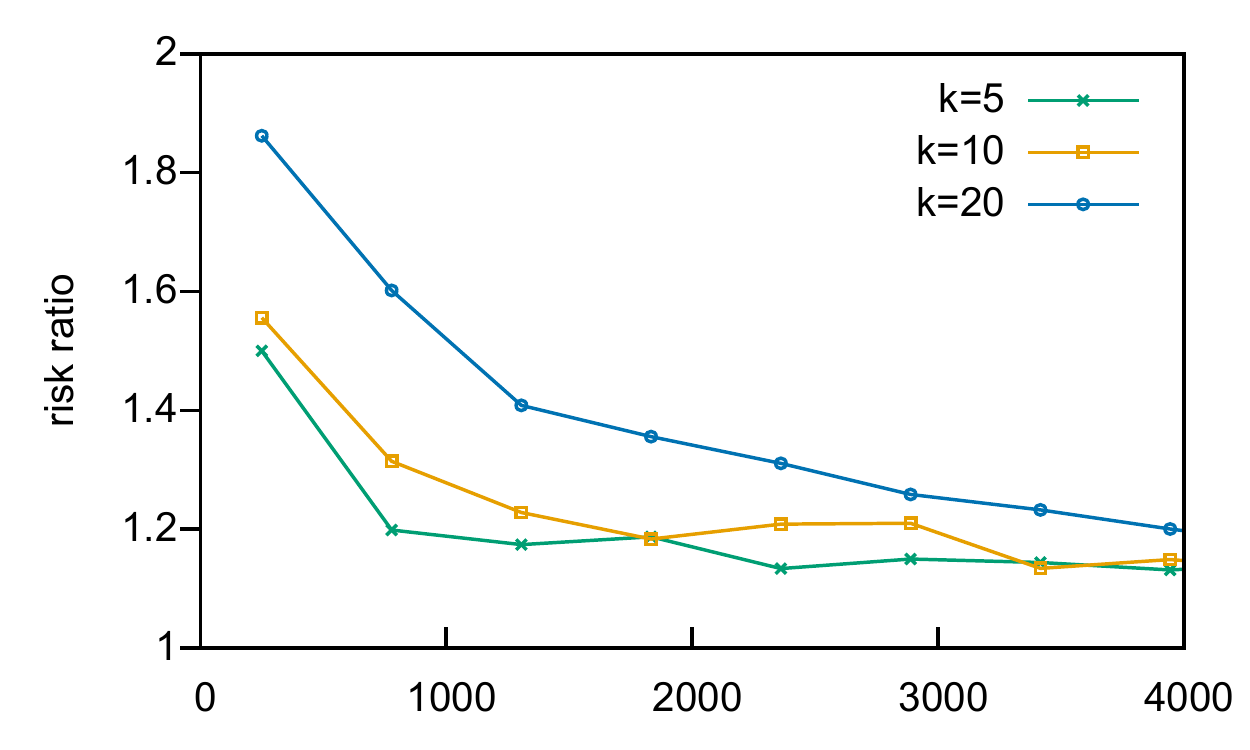}\\
  \includegraphics[width = 0.49\columnwidth]{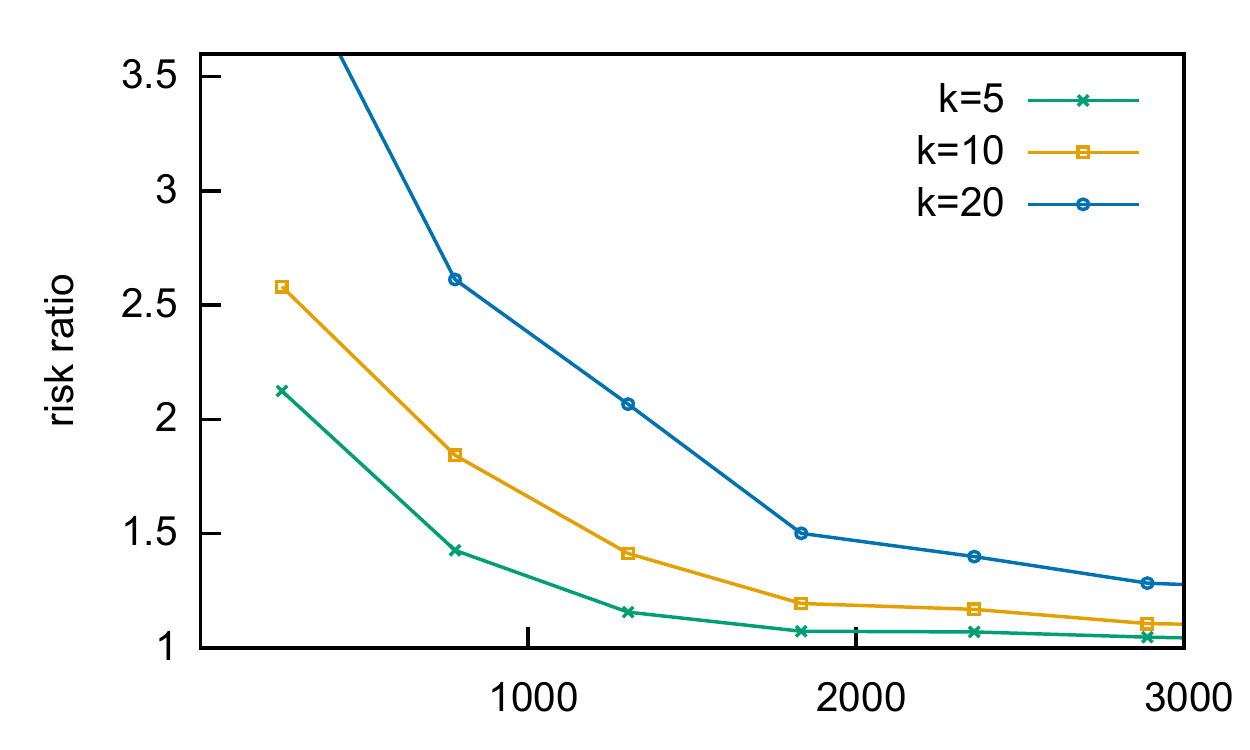}\\
  \caption{Risk ratio between \algname\ with $k$-medoids and offline $k$-medoids as a function of the stream size, for various values of $k$. Top left:\texttt{MNIST}, Top right: \texttt{Covertype}, Bottom: \texttt{Census}.}
  \label{fig:exp}
  \end{center}
\end{figure}

\section{Discussion} \label{sec:discussion}
In this work, we obtained an approximation factor which is twice that of the sample-based offline algorithm in the no-substitution setting. We showed that when disregarding computational considerations, the factor of 2 can be removed. It is an open question whether there is an efficient no-substitution algorithm with the same approximation factor as the best efficient offline algorithm. 
\optalg\ obtains an improved approximation factor by requiring that only centers with many possible choices of other centers are selected. This is related to notions of stability, or robustness, which have been previously studied for clustering algorithms in other contexts \cite[see, e.g.,][]{lange2004stability, ackerman2013clustering}, and more generally for learning algorithms \citep{bousquet2002stability}. The relationship between stability of algorithms and success in the no-substitution setting is an interesting direction for future research.

\paragraph{Acknowledgements}
This work was supported in part by the Israel Science Foundation (grant No. 555/15), and by the Lynn and Williams Frankel Center for Computer Science at Ben-Gurion University.

\bibliographystyle{abbrvnat}
\bibliography{Mybib}

\begin{thebibliography}{35}
\providecommand{\natexlab}[1]{#1}
\providecommand{\url}[1]{\texttt{#1}}
\expandafter\ifx\csname urlstyle\endcsname\relax
  \providecommand{\doi}[1]{doi: #1}\else
  \providecommand{\doi}{doi: \begingroup \urlstyle{rm}\Url}\fi

\bibitem[Ackerman et~al.(2013)Ackerman, Ben-David, Loker, and
  Sabato]{ackerman2013clustering}
M.~Ackerman, S.~Ben-David, D.~Loker, and S.~Sabato.
\newblock Clustering oligarchies.
\newblock In \emph{Artificial Intelligence and Statistics}, pages 66--74, 2013.

\bibitem[Ackermann et~al.(2012)Ackermann, M{\"a}rtens, Raupach, Swierkot,
  Lammersen, and Sohler]{ackermann2012streamkm++}
M.~R. Ackermann, M.~M{\"a}rtens, C.~Raupach, K.~Swierkot, C.~Lammersen, and
  C.~Sohler.
\newblock Streamkm++: A clustering algorithm for data streams.
\newblock \emph{Journal of Experimental Algorithmics (JEA)}, 17:\penalty0 2--4,
  2012.

\bibitem[Aggarwal and Yu(2005)]{aggarwal2005online}
C.~C. Aggarwal and P.~S. Yu.
\newblock Online analysis of community evolution in data streams.
\newblock In \emph{Proceedings of the 2005 SIAM International Conference on
  Data Mining}, pages 56--67. SIAM, 2005.

\bibitem[Ailon et~al.(2009)Ailon, Jaiswal, and Monteleoni]{ailon2009streaming}
N.~Ailon, R.~Jaiswal, and C.~Monteleoni.
\newblock Streaming k-means approximation.
\newblock In \emph{Advances in neural information processing systems}, pages
  10--18, 2009.

\bibitem[Babaioff et~al.(2007)Babaioff, Immorlica, Kempe, and
  Kleinberg]{babaioff2007knapsack}
M.~Babaioff, N.~Immorlica, D.~Kempe, and R.~Kleinberg.
\newblock A knapsack secretary problem with applications.
\newblock In \emph{Approximation, randomization, and combinatorial
  optimization. Algorithms and techniques}, pages 16--28. Springer, 2007.

\bibitem[Babaioff et~al.(2008)Babaioff, Immorlica, Kempe, and
  Kleinberg]{babaioff2008online}
M.~Babaioff, N.~Immorlica, D.~Kempe, and R.~Kleinberg.
\newblock Online auctions and generalized secretary problems.
\newblock \emph{ACM SIGecom Exchanges}, 7\penalty0 (2):\penalty0 7, 2008.

\bibitem[Badanidiyuru et~al.(2014)Badanidiyuru, Mirzasoleiman, Karbasi, and
  Krause]{badanidiyuru2014streaming}
A.~Badanidiyuru, B.~Mirzasoleiman, A.~Karbasi, and A.~Krause.
\newblock Streaming submodular maximization: Massive data summarization on the
  fly.
\newblock In \emph{Proceedings of the 20th ACM SIGKDD international conference
  on Knowledge discovery and data mining}, pages 671--680. ACM, 2014.

\bibitem[Bateni et~al.(2010)Bateni, Hajiaghayi, and
  Zadimoghaddam]{bateni2010submodular}
M.~Bateni, M.~Hajiaghayi, and M.~Zadimoghaddam.
\newblock Submodular secretary problem and extensions.
\newblock In \emph{Approximation, Randomization, and Combinatorial
  Optimization. Algorithms and Techniques}, pages 39--52. Springer, 2010.

\bibitem[Ben-David(2007)]{ben2007framework}
S.~Ben-David.
\newblock A framework for statistical clustering with constant time
  approximation algorithms for k-median and k-means clustering.
\newblock \emph{Machine Learning}, 66\penalty0 (2-3):\penalty0 243--257, 2007.

\bibitem[Bousquet and Elisseeff(2002)]{bousquet2002stability}
O.~Bousquet and A.~Elisseeff.
\newblock Stability and generalization.
\newblock \emph{Journal of machine learning research}, 2\penalty0
  (Mar):\penalty0 499--526, 2002.

\bibitem[Braverman et~al.(2016)Braverman, Feldman, and Lang]{braverman2016new}
V.~Braverman, D.~Feldman, and H.~Lang.
\newblock New frameworks for offline and streaming coreset constructions.
\newblock \emph{arXiv preprint arXiv:1612.00889}, 2016.

\bibitem[Charikar et~al.(2003)Charikar, O'Callaghan, and
  Panigrahy]{charikar2003better}
M.~Charikar, L.~O'Callaghan, and R.~Panigrahy.
\newblock Better streaming algorithms for clustering problems.
\newblock In \emph{Proceedings of the thirty-fifth annual ACM symposium on
  Theory of computing}, pages 30--39. ACM, 2003.

\bibitem[Chen(2009)]{chen2009coresets}
K.~Chen.
\newblock On coresets for k-median and k-means clustering in metric and
  euclidean spaces and their applications.
\newblock \emph{SIAM Journal on Computing}, 39\penalty0 (3):\penalty0 923--947,
  2009.

\bibitem[Dua and Graff(2017)]{Dua2019}
D.~Dua and C.~Graff.
\newblock {UCI} machine learning repository, 2017.
\newblock URL \url{http://archive.ics.uci.edu/ml}.

\bibitem[Dueck and Frey(2007)]{dueck2007non}
D.~Dueck and B.~J. Frey.
\newblock Non-metric affinity propagation for unsupervised image
  categorization.
\newblock In \emph{2007 IEEE 11th International Conference on Computer Vision},
  pages 1--8. IEEE, 2007.

\bibitem[Feldman et~al.(2011)Feldman, Naor, and Schwartz]{feldman2011improved}
M.~Feldman, J.~S. Naor, and R.~Schwartz.
\newblock Improved competitive ratios for submodular secretary problems.
\newblock In \emph{Approximation, Randomization, and Combinatorial
  Optimization. Algorithms and Techniques}, pages 218--229. Springer, 2011.

\bibitem[Guha et~al.(2000)Guha, Mishra, Motwani, and
  O'Callaghan]{guha2000clustering}
S.~Guha, N.~Mishra, R.~Motwani, and L.~O'Callaghan.
\newblock Clustering data streams.
\newblock In \emph{Foundations of computer science, 2000. proceedings. 41st
  annual symposium on}, pages 359--366. IEEE, 2000.

\bibitem[Hadi et~al.(2006)Hadi, Essannouni, and Thami]{hadi2006video}
Y.~Hadi, F.~Essannouni, and R.~O.~H. Thami.
\newblock Video summarization by k-medoid clustering.
\newblock In \emph{Proceedings of the 2006 ACM symposium on Applied computing},
  pages 1400--1401. ACM, 2006.

\bibitem[Hoeffding(1963)]{Hoeffding63}
W.~Hoeffding.
\newblock Probability inequalities for sums of bounded random variables.
\newblock \emph{Journal of the American Statistical Association}, 58\penalty0
  (301):\penalty0 13--30, 1963.

\bibitem[Kesselheim and T{\"o}nnis(2017)]{kesselheim2017submodular}
T.~Kesselheim and A.~T{\"o}nnis.
\newblock Submodular secretary problems: Cardinality, matching, and linear
  constraints.
\newblock In \emph{Approximation, Randomization, and Combinatorial
  Optimization. Algorithms and Techniques (APPROX/RANDOM 2017)}. Schloss
  Dagstuhl-Leibniz-Zentrum fuer Informatik, 2017.

\bibitem[Lang(2018)]{lang2018online}
H.~Lang.
\newblock Online facility location against at-bounded adversary.
\newblock In \emph{Proceedings of the Twenty-Ninth Annual ACM-SIAM Symposium on
  Discrete Algorithms}, pages 1002--1014. Society for Industrial and Applied
  Mathematics, 2018.

\bibitem[Lange et~al.(2004)Lange, Roth, Braun, and Buhmann]{lange2004stability}
T.~Lange, V.~Roth, M.~L. Braun, and J.~M. Buhmann.
\newblock Stability-based validation of clustering solutions.
\newblock \emph{Neural computation}, 16\penalty0 (6):\penalty0 1299--1323,
  2004.

\bibitem[Lattanzi and Vassilvitskii(2017)]{lattanzi2017consistent}
S.~Lattanzi and S.~Vassilvitskii.
\newblock Consistent k-clustering.
\newblock In \emph{International Conference on Machine Learning}, pages
  1975--1984, 2017.

\bibitem[LeCun et~al.(1998)LeCun, Bottou, Bengio, Haffner,
  et~al.]{lecun1998gradient}
Y.~LeCun, L.~Bottou, Y.~Bengio, P.~Haffner, et~al.
\newblock Gradient-based learning applied to document recognition.
\newblock \emph{Proceedings of the IEEE}, 86\penalty0 (11):\penalty0
  2278--2324, 1998.

\bibitem[Leung and Leckie(2005)]{leung2005unsupervised}
K.~Leung and C.~Leckie.
\newblock Unsupervised anomaly detection in network intrusion detection using
  clusters.
\newblock In \emph{Proceedings of the Twenty-eighth Australasian conference on
  Computer Science-Volume 38}, pages 333--342. Australian Computer Society,
  Inc., 2005.

\bibitem[Liberty et~al.(2016)Liberty, Sriharsha, and
  Sviridenko]{liberty2016algorithm}
E.~Liberty, R.~Sriharsha, and M.~Sviridenko.
\newblock An algorithm for online k-means clustering.
\newblock In \emph{2016 Proceedings of the Eighteenth Workshop on Algorithm
  Engineering and Experiments (ALENEX)}, pages 81--89. SIAM, 2016.

\bibitem[Maurer and Pontil(2009)]{maurer2009empirical}
A.~Maurer and M.~Pontil.
\newblock Empirical bernstein bounds and sample variance penalization.
\newblock \emph{arXiv preprint arXiv:0907.3740}, 2009.

\bibitem[Meyerson(2001)]{meyerson2001online}
A.~Meyerson.
\newblock Online facility location.
\newblock In \emph{Foundations of Computer Science, 2001. Proceedings. 42nd
  IEEE Symposium on}, pages 426--431. IEEE, 2001.

\bibitem[Nasraoui et~al.(2007)Nasraoui, Cerwinske, Rojas, and
  Gonzalez]{nasraoui2007performance}
O.~Nasraoui, J.~Cerwinske, C.~Rojas, and F.~Gonzalez.
\newblock Performance of recommendation systems in dynamic streaming
  environments.
\newblock In \emph{Proceedings of the 2007 SIAM International Conference on
  Data Mining}, pages 569--574. SIAM, 2007.

\bibitem[Ng et~al.(2006)Ng, Ong, Foong, Goh, and Nowinski]{ng2006medical}
H.~Ng, S.~Ong, K.~Foong, P.~Goh, and W.~Nowinski.
\newblock Medical image segmentation using k-means clustering and improved
  watershed algorithm.
\newblock In \emph{2006 IEEE Southwest Symposium on Image Analysis and
  Interpretation}, pages 61--65. IEEE, 2006.

\bibitem[Novikov(2019)]{Novikov2019}
A.~Novikov.
\newblock {PyClustering}: Data mining library.
\newblock \emph{Journal of Open Source Software}, 4\penalty0 (36):\penalty0
  1230, apr 2019.
\newblock \doi{10.21105/joss.01230}.
\newblock URL \url{https://doi.org/10.21105/joss.01230}.

\bibitem[Pedregosa et~al.(2011)Pedregosa, Varoquaux, Gramfort, Michel, Thirion,
  Grisel, Blondel, Prettenhofer, Weiss, Dubourg, Vanderplas, Passos,
  Cournapeau, Brucher, Perrot, and Duchesnay]{scikit-learn}
F.~Pedregosa, G.~Varoquaux, A.~Gramfort, V.~Michel, B.~Thirion, O.~Grisel,
  M.~Blondel, P.~Prettenhofer, R.~Weiss, V.~Dubourg, J.~Vanderplas, A.~Passos,
  D.~Cournapeau, M.~Brucher, M.~Perrot, and E.~Duchesnay.
\newblock Scikit-learn: Machine learning in {P}ython.
\newblock \emph{Journal of Machine Learning Research}, 12:\penalty0 2825--2830,
  2011.

\bibitem[Sabato and Hess(2018)]{SabatoHe18}
S.~Sabato and T.~Hess.
\newblock Interactive algorithms: Pool, stream and precognitive stream.
\newblock \emph{Journal of Machine Learning Research}, 18\penalty0
  (229):\penalty0 1--39, 2018.

\bibitem[Shepitsen et~al.(2008)Shepitsen, Gemmell, Mobasher, and
  Burke]{shepitsen2008personalized}
A.~Shepitsen, J.~Gemmell, B.~Mobasher, and R.~Burke.
\newblock Personalized recommendation in social tagging systems using
  hierarchical clustering.
\newblock In \emph{Proceedings of the 2008 ACM conference on Recommender
  systems}, pages 259--266. ACM, 2008.

\bibitem[Zhang et~al.(1997)Zhang, Ramakrishnan, and Livny]{zhang1997birch}
T.~Zhang, R.~Ramakrishnan, and M.~Livny.
\newblock Birch: A new data clustering algorithm and its applications.
\newblock \emph{Data Mining and Knowledge Discovery}, 1\penalty0 (2):\penalty0
  141--182, 1997.

\end{thebibliography}

\clearpage
\appendix

\twocolumn[
\begin{center}
%  \hrule
%  \vspace{0.5em}
  AISTATS 2020 Supplementary Material\\
  \vspace{1em}
  {\Large \textbf{\papertitle}}\\
  \vspace{0.5em}
  {\large Tom Hess and Sivan Sabato}\\
%  \hrule
%  \vspace{1em}
  \vspace{0.5em}
  \hrule
\end{center}
\vspace{1em}
]

\section{Bernstein and empirical Bernstein inequalities}

\begin{proof}[Proof of \lemref{empbernstein}]
We use the Empirical Bernstein inequality of \citep{maurer2009empirical}. This inequality states that for $\hat{\sigma}^2 :=\frac{1}{2n(n-1)}\sum_{i,j \in [n], i \neq j}(Y_i-Y_j)^2$, with a probability at least $1-\delta$, we have
\[
\hat{\mu} - \mu \leq \frac{7 \ln (\frac{2}{\delta})}{3(n-1)} + \sqrt{\frac{2\hat{\sigma}^2 \ln(\frac{2}{\delta})}{n}}.
\]
We have $\hat{\sigma}^2 \leq \frac{n}{2(n-1)} \frac{1}{n^2}\sum_{i,j \in [n]} (Y_i - Y_j)^2 = \frac{n}{2(n-1)} \E[(Y - Y')^2]$, where $Y,Y'$ are drawn independently and uniformly from the fixed sample $Y_1,\ldots,Y_n$. Since $\E[(Y - Y')^2] \leq 2\E[Y^2]$, and $Y \in [0,1]$, we have $\hat{\sigma}^2 \leq \frac{n}{n-1}\E[Y] \equiv \frac{n}{n-1}\hat{\mu}$. Therefore,
\[
  \hat{\mu} - \mu \leq \frac{7 \ln (\frac{2}{\delta})}{3(n-1)} + \sqrt{\frac{2\hat{\mu} \ln(\frac{2}{\delta})}{n-1}}.
  \]
  If $\hat{\mu} = a \ln(\frac{2}{\delta})/(n-1)$ for $a \geq 16$, then the RHS is at most \[
    (7/3 + \sqrt{2a})\ln(\frac{2}{\delta})/(n-1)\leq a/2 \cdot \ln(\frac{2}{\delta})/(n-1) \leq \hat{\mu}/2.
    \]
\end{proof}

\begin{proof}[Proof of \lemref{bernstein}]
  Let $\sigma^2 = \Var[Y_i]$. By Bernstein's inequality \citep{Hoeffding63} (See, e.g., \citealt{maurer2009empirical} for the formulation below),
  \[
  \mu - \hat{\mu} \leq \frac{\ln(\frac{1}{\delta})}{3n}+\sqrt{\frac{2\sigma^2\ln(\frac{1}{\delta})}{n}}.
\]
Since $Y_i$ are supported on $[0,1]$, we have $\sigma^2 \leq \mu$. 
Since $\mu = a \ln(\frac{1}{\delta})/n$ for $a \geq 10$, we have that the RHS is equal to $(1/3 + \sqrt{2a})\ln(\frac{1}{\delta})/n\leq a/2 \cdot \ln(\frac{1}{\delta})/n \leq \mu/2$. The statement of the lemma follows. 
\end{proof}

\section{Tightness of multiplicative factor of \algname}

\begin{proof}[Proof of \thmref{2example}]
  We define a weighted undirected graph $G=(V,E,W)$, and let $(\cX,\rho)$ be a metric space such that $\cX = V$ and $\rho(u,v)$ is the length of the shortest path in the graph between $u$ and $v$. $G$, which is illustrated in \figref{example2fnew}, is formally defined as follows. The set of nodes is $V := U \cup Y \cup \{o,v\}$, where $U :=[0,1]$ and  $Y:=[3,4]$. The set of edges is
  \[
    E:=\{\{u,o\} \mid u \in U \} \cup   \{\{v,y\} \mid y \in Y \} \cup \{ \{o,v\}\}.
  \]
  Denote $m_1 := m/2$, and let $\eta := 1/(4m_1)$. The weight function $W$ assigns a weight of $1$ to all edges except for those that have a node in $Y$ as an endpoint, which are assigned a weight of $2-\eta$. 
  \begin{figure}[H]
   \centering
   \includegraphics[height=1.3in]{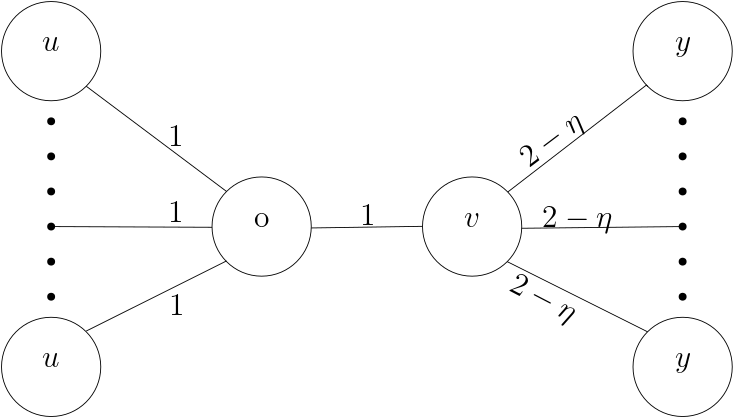}
   \caption{Illustration of the graph $G$ which defines the metric space.} 
   \label{fig:example2fnew}
\end{figure}

Define the distribution $P$ over $\cX$ such that $P(o) = 0$, $P(v) = 1/m_1$, $P(Y) = 2q$, with a uniform conditional distribution over $Y$. Lastly,  $P(U) = 1-2q-\frac{1}{m_1}$, with a uniform conditional distribution over $U$. Note that the latter is positive for a large enough $m_1$, since $q(m_1) \rightarrow 0$. 
 
Let $S \sim P^m$ be the i.i.d.~sample used as an input sequence to \algname,  and set $k=1$. Let $S_1$ be the sample observed in the first phase of \algname, of size $m_1$. Define the following events:
\begin{enumerate}
\item   $E_1:=\{o \notin  S_1\}$. 

\item   $E_2:=\{v\text{ appears in }S_1\}.$ 
\item $E_3:=  \{$at least $q m_1$ of the samples in $S_1$ are from $Y \}$. 
 
\end{enumerate}  
First, observe that all these events occur together with a positive probability, as follows. $\P[E_1] = 1$ since $P(o) = 0$. For $E_2$, we have
\[
\P[E_2] > 1-(1-\frac{1}{m_1})^{m_1} >\frac{1}{2}.
\]
For $E_3$, note that the probability mass of $Y$ is $2q$, Apply \lemref{bernstein} with $\mu = 2q$, $n = m_1$ and a confidence value of $1/4$. By the assumption of the theorem, for sufficiently small $\delta$, we have $q \geq 5\log(4)/m_1$. Therefore, \lemref{bernstein} implies that $P[E_3] \geq 3/4$. 

It follows that $\P[E_1\wedge E_2 \wedge E_3] \geq 1/4$.

Now, assume that all the events above hold. By $E_1$, $o$ does not appear in $S_1$, and by $E_2$, $v$ appears in $S_1$. We show that out of the points in $S_1$, the $1$-clustering $\{v\}$ has the best empirical risk. The only other options in $S_1$ are centers from $Y$ or from $U$. For a center $u \in U$ from $S_1$, note that with a probability $1$, it does not have additional copies in $S_1$. Its distance from all other $u' \in U$ is the same as that of $v$, while its distance from points in $Y$ and from $v$ is larger. Thus, $R(S_1, \{u\}) > R(S_1, \{v\})$. For a center $y \in Y$, it too does not have additional copies in $S_1$. Its distance to all other points is larger than that of $v$. Thus, $R(S_1, \{y\}) > R(S_1, \{v\})$. 
Therefore, $v$ has the best empirical risk on $S_1$. Thus, $\cA(S_1)$ returns the $1$-clustering $\{v\}$.

By $E_3$, the number of instances of vertices from $Y$ is at least $qm_1$. Since the points in $Y$ are the closest to $v$ in $S_1$, we have $y' := \qpoint_{S_1}(v,q) \in Y$. Therefore, $\qball(v,y')=\{ v\} \cup Y$.  It follows that \algname\ selects as a center the first element from $\{ v\} \cup Y$ that it observes in the second phase. With a probability $\frac{2q}{2q+\frac{1}{m_1}}$, the first element that  \algname\ observes from $\{ v\} \cup Y$ is in $Y$. Since $q \geq 1/m_1$, this probability is at least $2/3$. Thus, the output center of \algname\ is from $Y$ with a constant probability.

However, the risk of this clustering is large:
\begin{align*}
  &R(P,\{y\})\\
  &=(4-\eta)(1-2q-\frac{1}{m_1})+(2-\eta)\frac{1}{m_1}+(4-\eta)2q.
\end{align*}
For large $m$, we have $m_1 \rightarrow \infty$. In addition, $q, \eta \rightarrow 0$. Hence, $R(P,\{y\}) \rightarrow 4$. 
In contrast, the risk using $o$ as a center is small:
\begin{align*}
  R(P,\{o\})&=( 1-2q-\frac{1}{m_1})+\frac{1}{m_1}+(3-\eta)2q.
\end{align*}
This approaches $1$ for large $m$. 
Therefore, for $m   \rightarrow \infty$, $R(P,\{y\})/R(P,\{o\}) \rightarrow 4$. Since $\{y\}$ is the output of \algname\ with a constant probability, the multiplicative factor obtained by \algname\ cannot be smaller than $4 = 2\beta$ in this case.
\end{proof}

\section{Full results of experiments}
The results of the experiments for large stream sizes with the $k$-medoids as the black box are reported in \figref{expfull}. The results for the BIRCH black-box are reported in \figref{expbirch} and in \figref{expfullbirch}. For the $k$-medoids black box, the risk ratios for large stream sizes are in the following ranges: \texttt{MNIST} $1.02-1.04$, \texttt{Covertype} $1.04-1.08$, \texttt{Census} $1-1.04$. For the BIRCH black box, the risk ratios for large stream sizes are in the following ranges: \texttt{MNIST} $1.03-1.04$, \texttt{Covertype} $1.05-1.1$, \texttt{Census} $1-1.02$. Thus, the risk ratio converges to a ratio very close to $1$.
\begin{figure}[t]
  \begin{center}
  \includegraphics[width = \columnwidth]{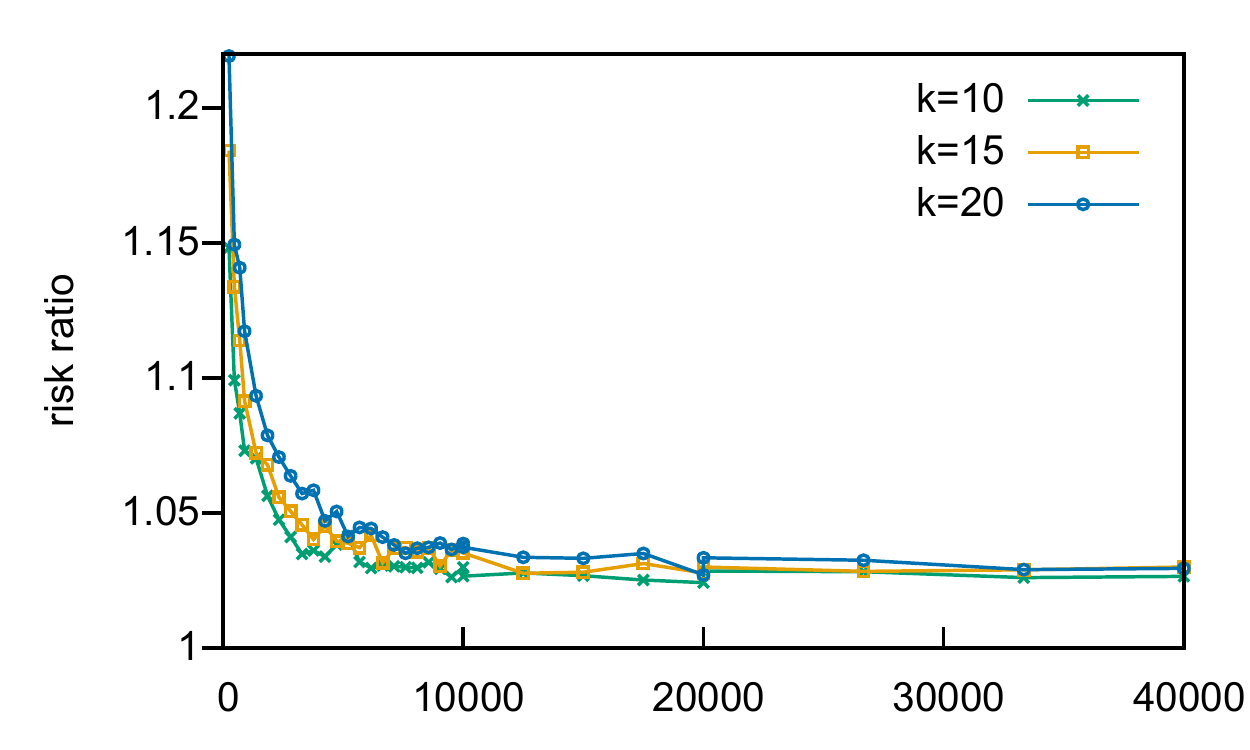}\\
  \includegraphics[width = \columnwidth]{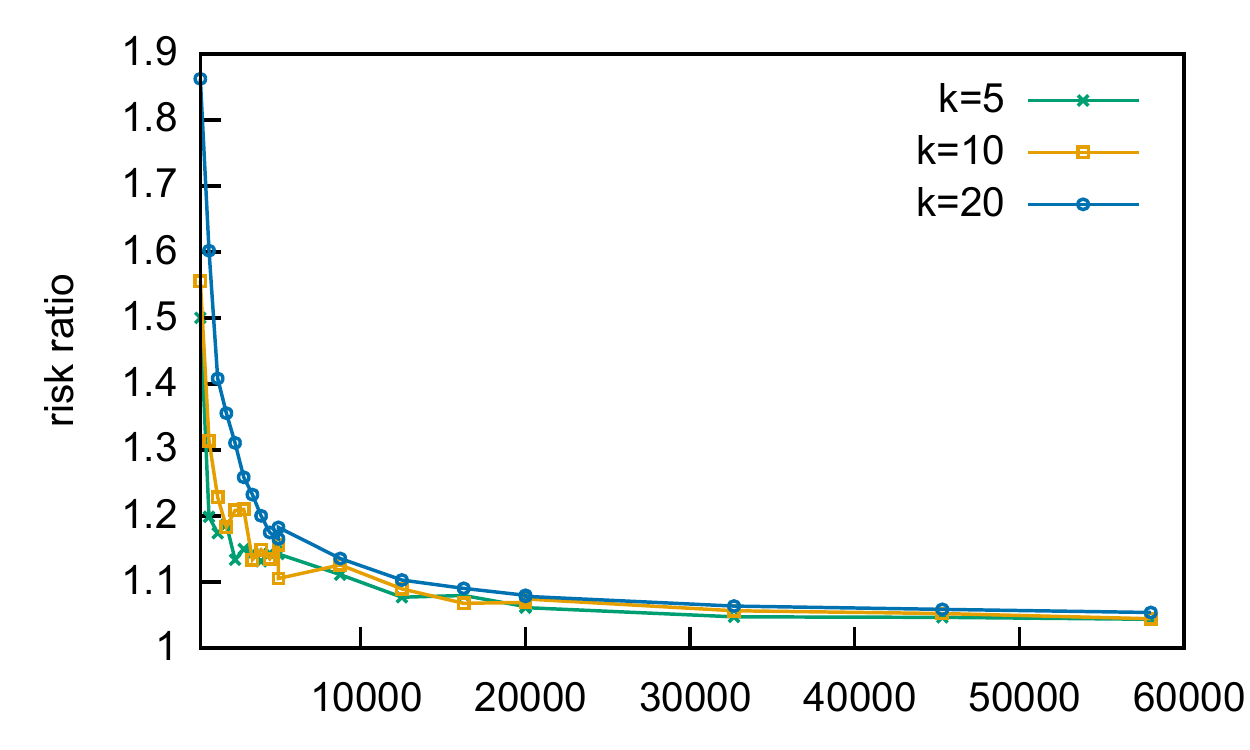}\\
  \includegraphics[width = \columnwidth]{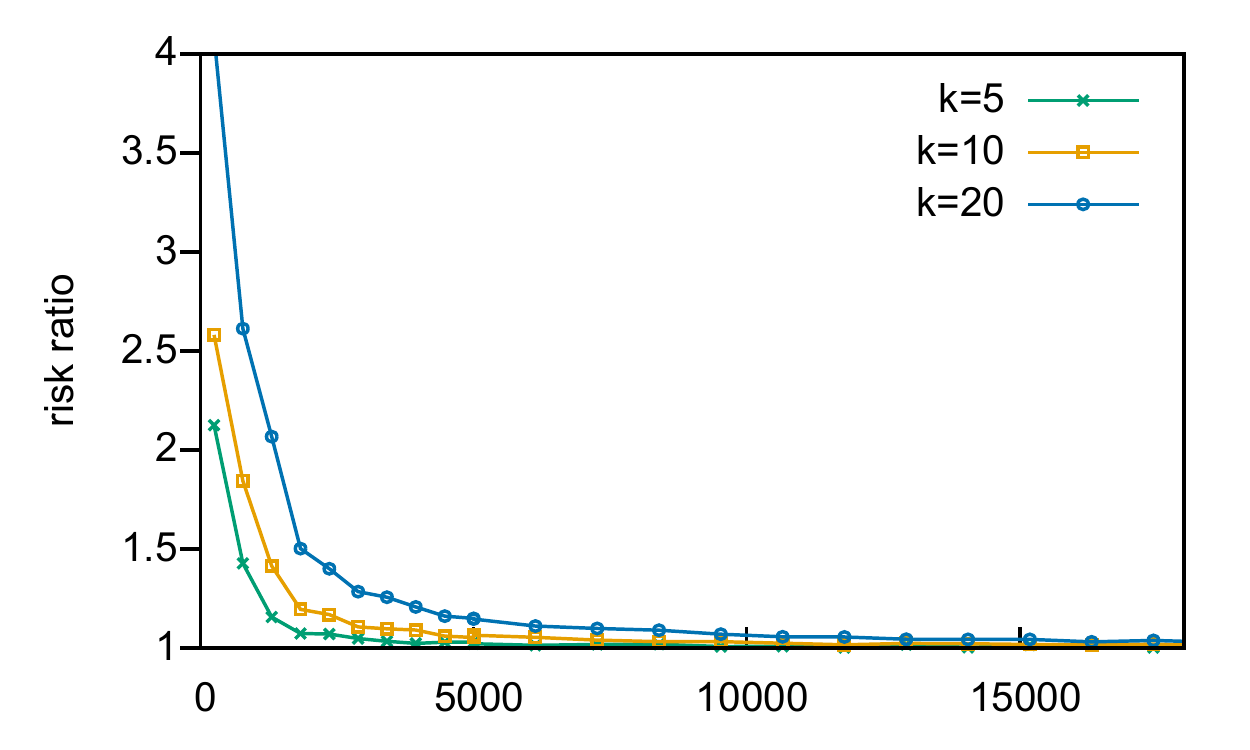}\\
  \caption{Risk ratio between \algname\ with $k$-medoids and offline $k$-medoids for large stream sizes, as a function of the stream size, for various values of $k$. Top to bottom: \texttt{MNIST, Covertype, Census}.}
  \label{fig:expfull}
  \end{center}
\end{figure}

\begin{figure}[t]
  \begin{center}
  \includegraphics[width = \columnwidth]{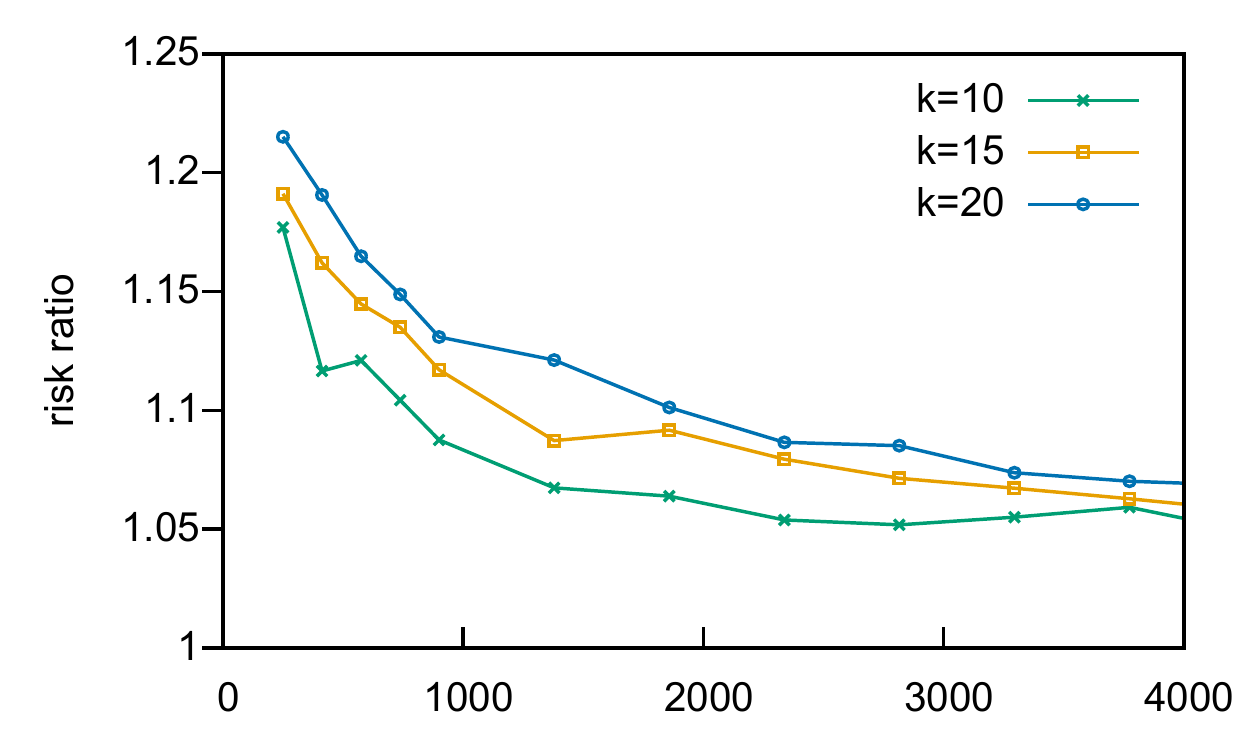}\\
  \includegraphics[width = \columnwidth]{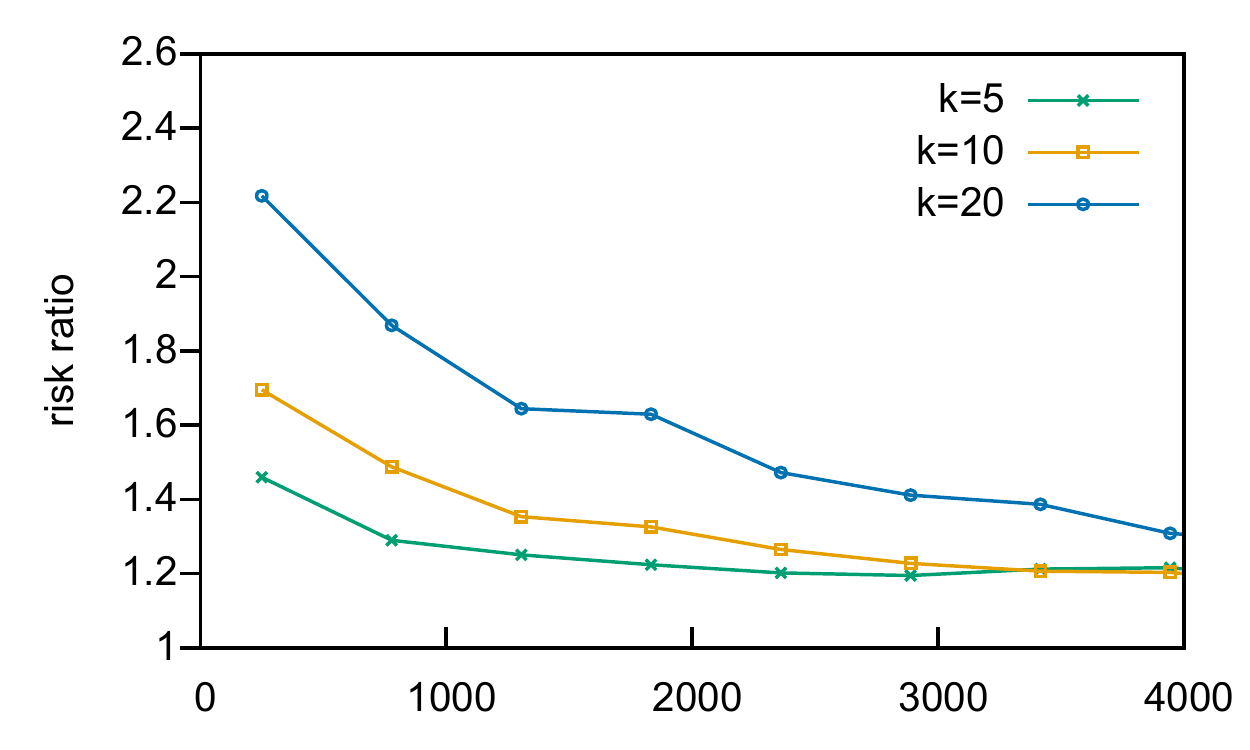}\\
  \includegraphics[width = \columnwidth]{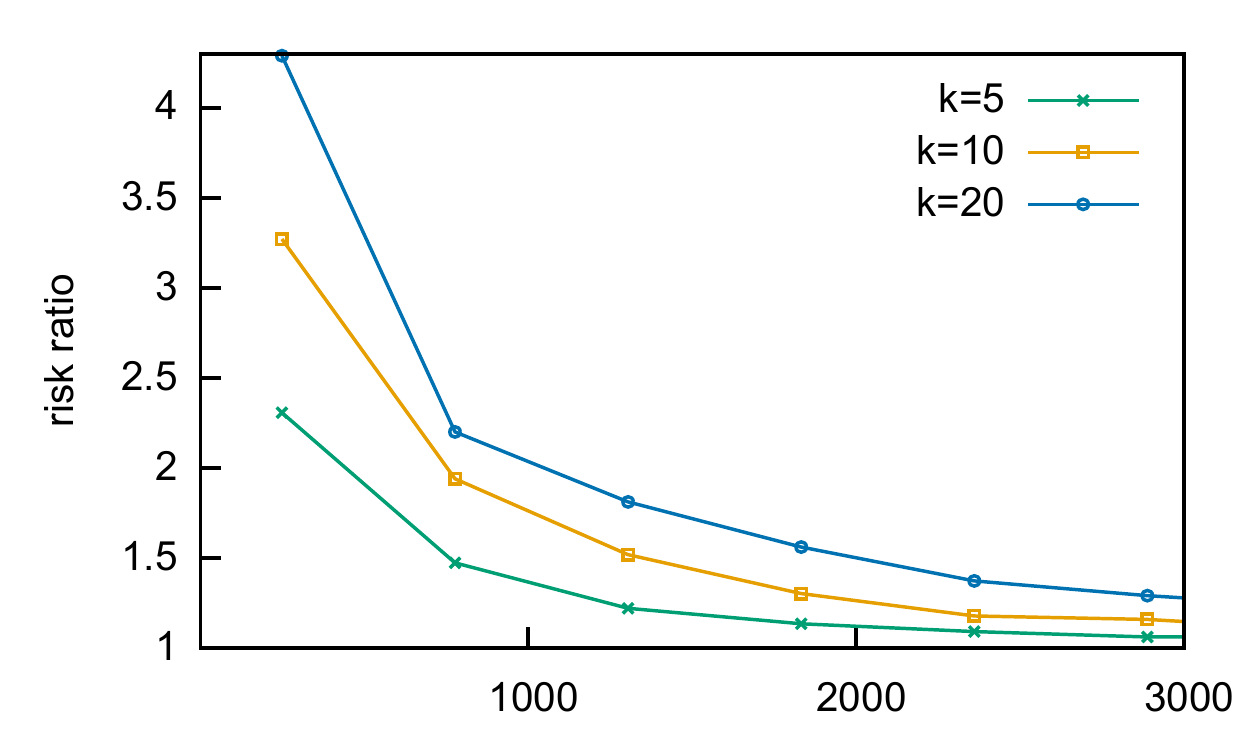}\\
  \caption{Risk ratio between \algname\ with BIRCH and offline BIRCH as a function of the stream size, for various values of $k$. Top to bottom: \texttt{MNIST, Covertype, Census}.}
  \label{fig:expbirch}
  \end{center}
\end{figure}
\vfill

\begin{figure}[t]
  \begin{center}
  \includegraphics[width = \columnwidth]{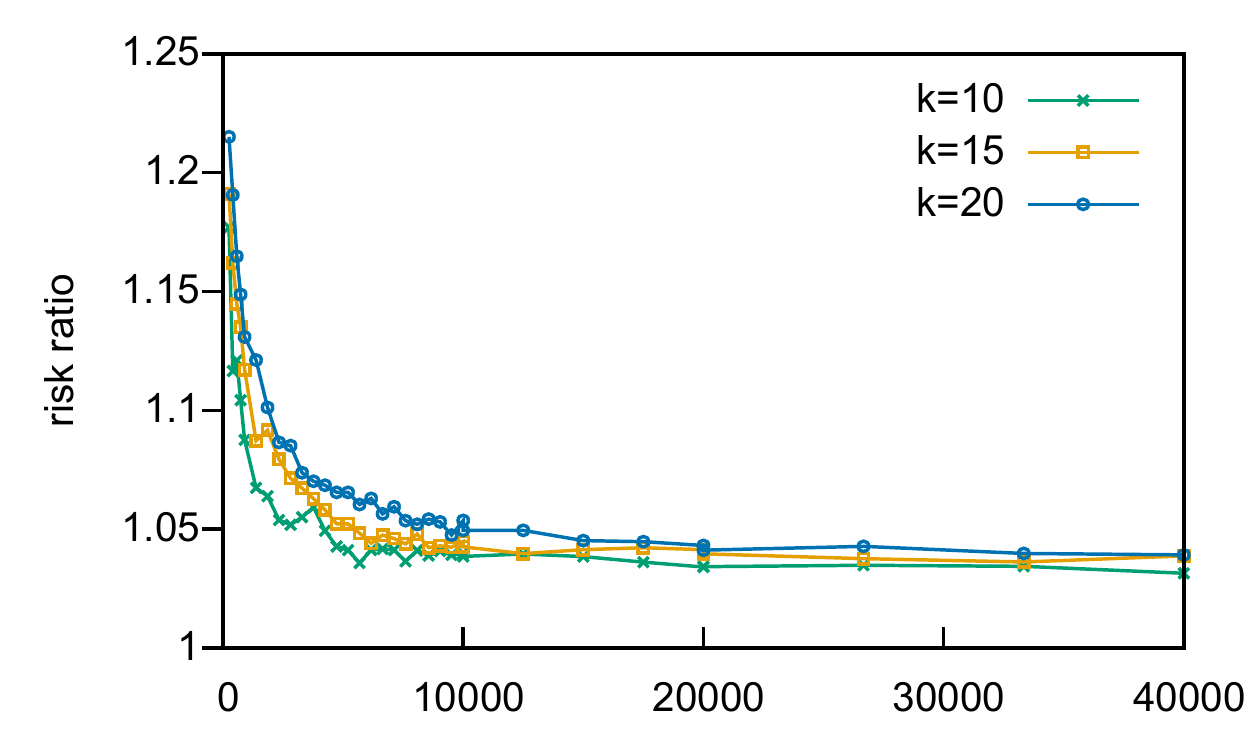}\\
  \includegraphics[width = \columnwidth]{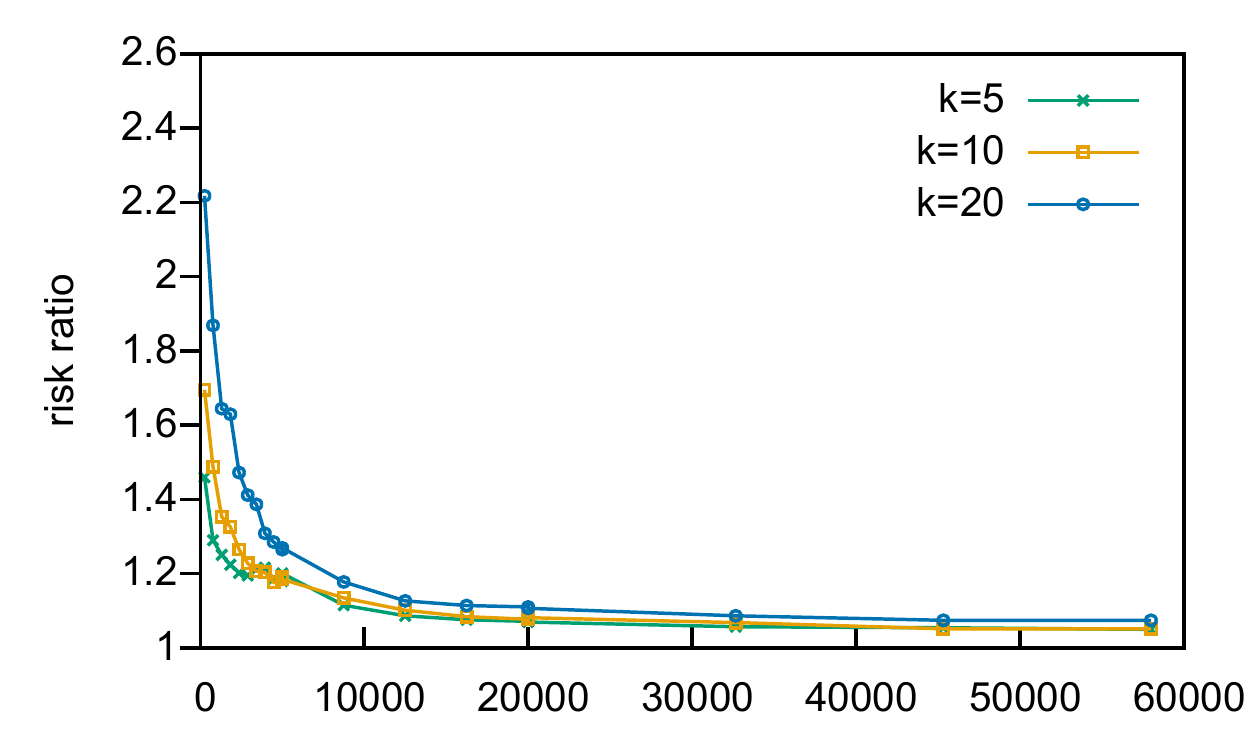}\\
  \includegraphics[width = \columnwidth]{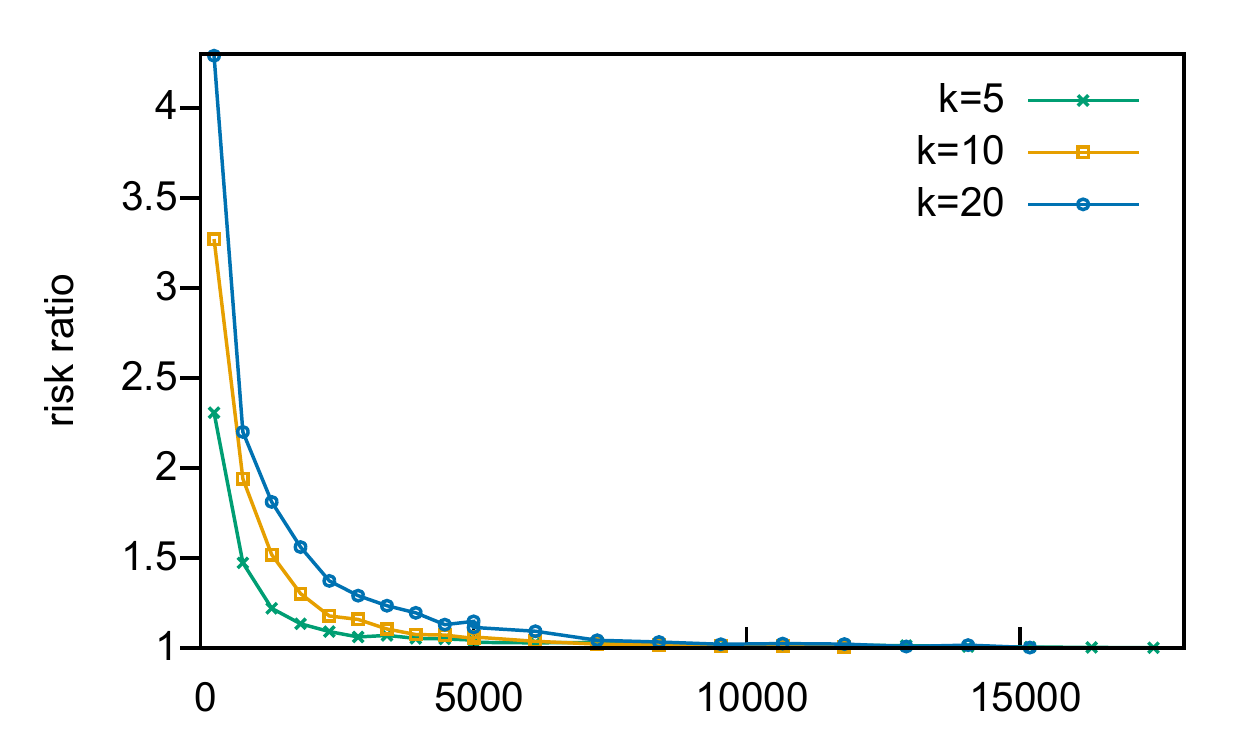}\\
  \caption{Risk ratio between \algname\ with BIRCH and offline BIRCH for large stream sizes, as a function of the stream size, for various values of $k$. Top to bottom: \texttt{MNIST, Covertype, Census}.}
  \label{fig:expfullbirch}
  \end{center}
\end{figure}
\vfill

\end{document}